\documentclass[conference]{IEEEtran}
\IEEEoverridecommandlockouts
\usepackage{cite}
\usepackage{amsmath,amssymb,amsfonts}
\usepackage{algorithmic}
\usepackage{graphicx}
\usepackage{textcomp}
\usepackage{xcolor}

\usepackage{booktabs}
\usepackage{multirow}

\usepackage{bbold}

\usepackage{caption}
\usepackage{subcaption}

\usepackage{float}
\usepackage{amsthm}

\newtheorem{theorem}{Theorem}
\newtheorem{mydef}{Definition}

\def\BibTeX{{\rm B\kern-.05em{\sc i\kern-.025em b}\kern-.08em
    T\kern-.1667em\lower.7ex\hbox{E}\kern-.125emX}}
\begin{document}

\title{Transferable Attack against Face Swapping in an Extended Space}

\author{
    \IEEEauthorblockN{Mingzhi Lyu\thanks{$^{\dagger}$  These authors contributed equally to this work.}$^{\dagger}$}
    \IEEEauthorblockA{
    \textit{Nanyang Technological University}\\
    Singapore \\
    lyum0002@e.ntu.edu.sg}
\and
    \IEEEauthorblockN{Yi Huang$^{\dagger}$  }
    \IEEEauthorblockA{
    \textit{Nanyang Technological University}\\
    Singapore \\
    shellbyhuang@gmail.com}

\and
    \IEEEauthorblockN{Jun Xie}
    \IEEEauthorblockA{
    \textit{Nanyang Technological University}\\
    Singapore \\
    jun.xie@ntu.edu.sg} 
\and
    \IEEEauthorblockN{ } 
    \IEEEauthorblockA{ } 
\and
    \IEEEauthorblockN{ } 
    \IEEEauthorblockA{ } 
\and 
    \IEEEauthorblockN{Zihao Zhao}
    \IEEEauthorblockA{
    \textit{Nanyang Technological University}\\
    Singapore \\
    zihao005@e.ntu.edu.sg}
\and
    \IEEEauthorblockN{Hong Xu}
    \IEEEauthorblockA{
    \textit{Nanyang Technological University}\\
    Singapore \\
    xuhong@ntu.edu.sg}
\and
    \IEEEauthorblockN{Kong Wai-Kin Adams}
    \IEEEauthorblockA{
    \textit{Nanyang Technological University}\\
    Singapore \\
    adamskong@ntu.edu.sg}

}
\maketitle
\begingroup\renewcommand{\thefootnote}{}\footnotetext{\textcopyright{} 2025 IEEE. Personal use of this material is permitted. Permission from IEEE must be obtained for all other uses, in any current or future media, including reprinting/republishing this material for advertising or promotional purposes, creating new collective works, for resale or redistribution to servers or lists, or reuse of any copyrighted component of this work in other works.}\endgroup

\begin{abstract}
   Although deep Face Swapping (FS) models may benefit the entertainment industry, they pose severe threats to privacy and security. Existing protections, including deepfake detection and adversarial perturbation, are either passive responses or ineffective to unseen subject-agnostic FS models. In this paper, we propose a transferable attack against subject-agnostic FS models named Additive Identity attack based on a Relighting function (AIR). AIR leverages reillumination and additive perturbations to mislead the identity extraction modules in subject-agnostic FS models. By using these two types of perturbations simultaneously, the attack space is extended such that stronger but more visually natural adversarial examples can be identified. To further enhance the visual quality while preserving the effectiveness of the attack, an adaptive translation-invariant operation and an illumination control scheme are designed for AIR. Unlike other methods, AIR does not require a surrogate FS model to achieve high transferability. In addition, a mathematical proof is given for the extension of the attack space. Extensive experiments using 1000 image pairs across various state-of-the-art subject-agnostic FS models, including GAN and diffusion-based FS models, show that AIR surpasses all existing attacks in terms of both attack success rate and image quality.
\end{abstract}


\begin{center}
\small Accepted by IEEE International Conference on Multimedia and Expo (ICME) 2025. This arXiv version includes the appendix.
\end{center}

\section{Introduction}

FS refers to the process of swapping a person’s face (target face) with the face of another (source face). In recent years, advanced FS models, especially subject-agnostic FS models, have raised significant public concerns due to their potential misuse. Examples of potential misuse include blackmailing a victim\footnote{https://www.entrepreneur.com/article/414109} with fake materials and swaying elections by misrepresenting well-known politicians\footnote{https://www.brookings.edu/research/is-seeing-still-believing-the-deepfake-challenge-to-truth-in-politics/}. 
 
 To avoid serious consequences caused by these malicious applications, great efforts are being devoted to designing deepfake detectors \cite{ huang2022fakelocator}. 
 However, detectors can only spot fake content after it has been created and disseminated. Alternatively, some researchers proposed to use adversarial examples (AEs) \cite{huang2021cmua,ruiz2020disrupting,dong2021visually,aneja2022tafim, zhang2024dual} to mislead deepfake models into generating unsound results under the assumption of perfect knowledge. To safeguard face images against FS models in a black-box scenario, LaS-GSA \cite{yeh2021attack} was introduced. However, this query-based method has drawbacks:  1) users need access to the target FS model for each query, which is usually inaccessible, 2) no evidence shows that the generated AEs are transferable to protect against unseen models even using the queried model as surrogate in training, and 3) it requires many queries (40,000-70,000 in LaS-GSA). Alternatively, some \cite{huang2021initiative, dong2023restricted} utilized surrogate models and trained generators to create transferable AEs to defend against unseen deepfake models. However, their attacks do not specifically target subject-agnostic FS models, which are our focus due to their broad applicability. The substantial architectural differences between different face manipulation models restrict the effectiveness of their suggested attacks against subject-agnostic FS models.

Although some transfer-based attacks \cite{huang2021initiative} and \cite{dong2023restricted} have been proposed, there are several challenges in the context of attacking FS models.
 First, obtaining FS models can be difficult as developers may have concerns about sharing them publicly due to social and ethical reasons. Second, the large variation in training data and model architectures used by different FS models often limits the transferability of AEs between them. Our experiments show that AEs generated using white-box attacks on surrogate FS models perform poorly on target black-box models. Furthermore, striking a balance between the transferability and imperceptibility of adversarial perturbations is a significant challenge when attacking FS models. State-of-the-art techniques like TAIG \cite{huang2022transferable} have shown effectiveness in enhancing transferability but often introduce visible noise textures, as illustrated in Fig. \ref{fig.taig-diff}. Even with TAIG, FS outputs after attack still exhibit a low Attack Success Rate (ASR)
  of 0.12 when the perturbation budget $\epsilon$ is 0.02, and it increases to 0.49 as $\epsilon$ increases to 0.05. Increasing $\epsilon$ can improve ASR but introduces more noticeable noise, degrading image quality and affecting perception of normal faces. 

\begin{figure}
	\centering 

	\begin{subfigure}[b]{0.54\columnwidth}
		\centering
		\includegraphics[width=\columnwidth]{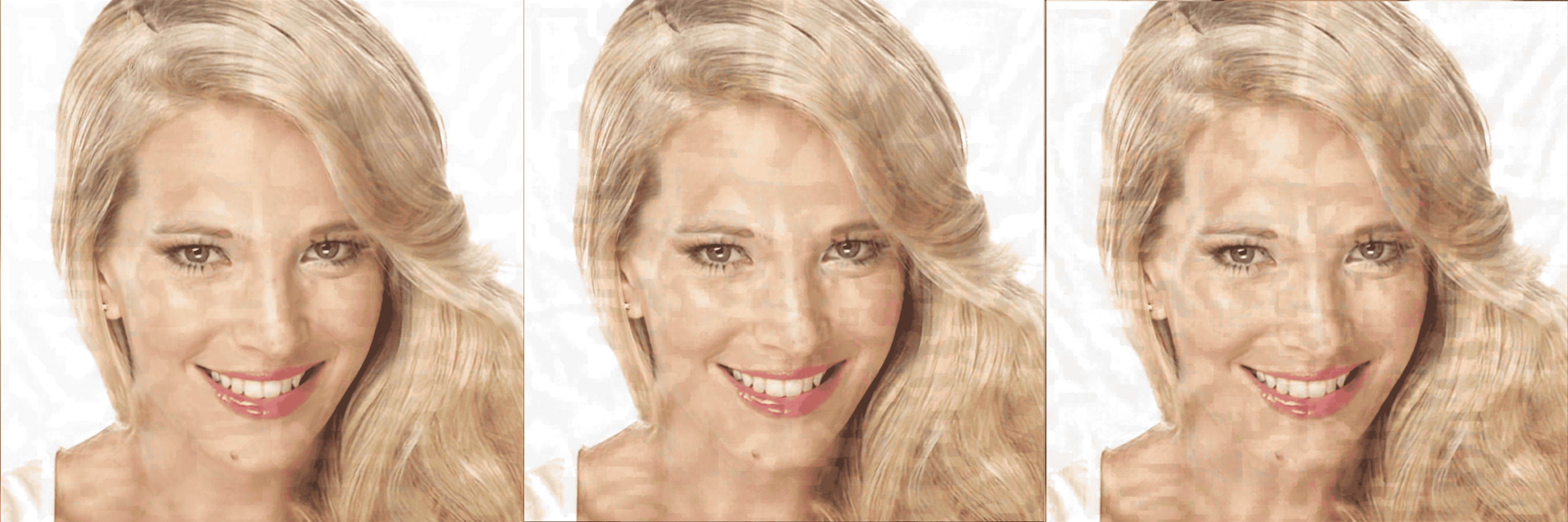}
		\caption{}
		\label{fig.taig-diff}
	\end{subfigure}
	\hfill
	\begin{subfigure}[b]{0.37\columnwidth}
		\centering
		\includegraphics[width=\columnwidth]{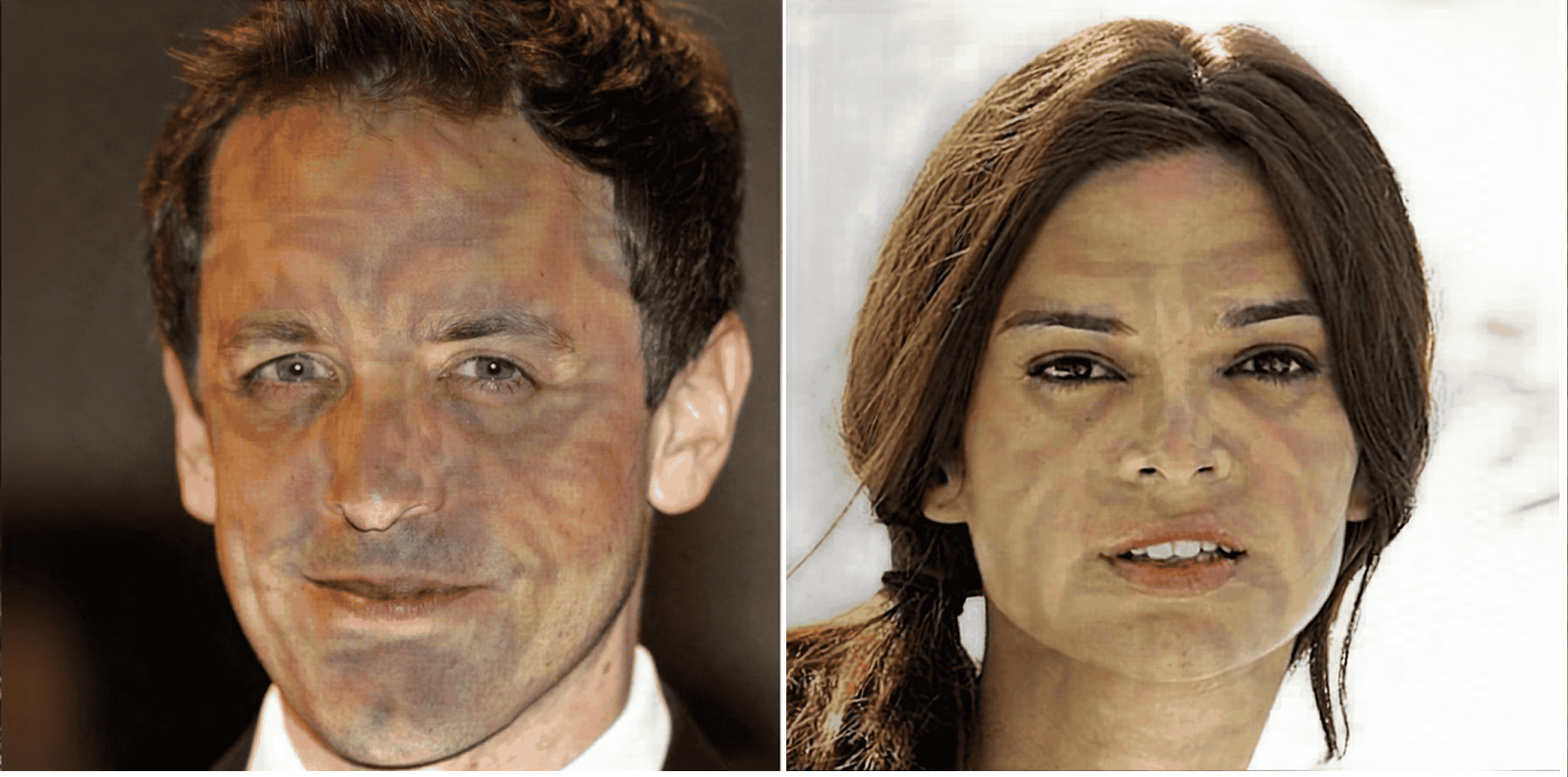}
		\caption{}
		\label{fig.ti}
	\end{subfigure}

	\caption{(a) AEs with TAIG under $\epsilon=0.02$ (left), $\epsilon=0.03$ (middle) and $\epsilon=0.05$ (right). (b) AEs with TI.}
		\label{Fig.textual-example} 
		\vspace{-0.6cm}
\end{figure}

To address these challenges, we propose the Additive Identity attack based on a Relighting function (AIR), which comprises two novel attack strategies: the Additive Identity Attack (AIA) and the Relighting Functional Attack (RFA). AIR features several key innovations:

\begin{enumerate}
    \item 
	Instead of directly employing FS models, AIR leverages face recognition models as surrogates, which circumvents the challenge of obtaining surrogate FS models while capitalizing on the reduced variability between face recognition models compared to full FS systems, thereby enhancing transferability.

    \item 
	By incorporating two types of perturbation, we expand the attack space, facilitating the discovery of AEs with superior transferability and lower perceptibility. We provide a theoretical justification for this expansion in the appendix's analysis section.


	\item 
	We evaluate AIR on multiple state-of-the-art subject-agnostic FS models across different frameworks, including GAN-based and diffusion-based models. To our best knowledge, this is the first study to assess transferable attack across GAN-based and diffusion-based FS frameworks. Our experimental results demonstrate the superiority of AIR in both attack performance and image quality, advancing the state-of-the-art in adversarial attacks on face swapping systems.
	
\end{enumerate}

\section{Related Work}

\textbf{Face Swapping Methods based on Deep Learning}. Deep FS models can be divided into subject-specific models and subject-agnostic models. The subject-specific models \cite{naruniec2020high} can only swap predefined source faces or require further training to swap other faces, restricting their practical applicability in real-world scenarios. In contrast, the subject-agnostic models \cite{li2020advancing,chen2020simswap,zhu2021one, kim2022diffface, zhao2023diffswap} overcome this drawback. Some FS models can achieve high-fidelity results on pairs of faces not contained in training sets. These models typically operate by disentangling facial identities from attributes such as pose, expression, and hairstyle using an encoder architecture. However, the versatility of subject-agnostic models poses significant risks, as they can be misused to swap any face onto any target, potentially leading to malicious applications. Our study therefore focuses on generating adversarial examples to protect source faces against these subject-agnostic models.

\textbf{Adversarial Attacks Against Deepfake}. In 2020, Ruiz \cite{ruiz2020disrupting} first used adversarial examples (AEs) to disrupt deepfake networks that were designed to manipulate attributes such as facial expression and hair color (Disrupt Attack). 
Dong and Xie \cite{dong2021visually} proposed three attacks on FS models: the transfer adversarial attack (TAA), which generates universal noise, and the Siamese adversarial attack (SAA) and latent Siamese attack (LSA), which generate image-specific noise. Huang \cite{huang2021cmua} introduced Cross-Model Universal Adversarial watermark (CMUA), a fusion of adversarial perturbations from attacks on multiple models. Concurrently, some researchers \cite{aneja2022tafim, zhang2024dual} used generative models to create AEs against FS. These methods are studied under a white-box setting assuming perfect knowledge of target models. Besides, some \cite{huang2021initiative, dong2023restricted} generated transferable adversarial examples using generators, focusing on face attribute editing and subject-specific FS models. However, a subject-specific FS model is limited to swapping portraits to a specific identity, restricting their applicability. Existing methods do not address subject-agnostic FS models, highlighting a need for solutions to protect online face images from such manipulations. Our study fills this gap with AIR, which effectively prevents manipulation by unseen subject-agnostic FS models.


\textbf{Adversarial Attacks Under Different Threat Models}. Adversarial attacks have evolved under various threat models in recent years. The additive threat model, where perturbations are added to individual input pixels with an $L_p$ norm constraints, has been widely used in attacks. 
However, this model often struggles with imperceptibility. Alternative threat models have emerged to address this issue. The spatial threat model alters pixel locations through operations like rotation and translation, as demonstrated by \cite{xiao2018spatially}. Laidlaw and Feizi \cite{laidlaw2019functional} introduced the functional threat model, which applies a single function to change pixel colors in a 3-dimensional color space. While this approach improves imperceptibility, it suffers from limited attack effectiveness due to a constrained attack space. To overcome these limitations, we propose a novel approach for attacking face swapping models that combines a relighting function with additive perturbations. By integrating these approaches, we aim to achieve a balance between effectiveness and imperceptibility of adversarial examples in a extended attack space.

\section{Methodology}

To safeguard face images without compromising their quality, we propose AIR, a novel approach that generates adversarial images that appear natural and preserve the original image's key content. 
AIR integrates two novel attack components as shown in Fig. \ref{Fig.Schematic}: Additive Identity Attack (AIA) and Relighting Functional Attack (RFA). AIA, operating under the additive threat model, incorporates a new Adaptive Translation-Invariant (ATI) operation, enhancing transferability without introducing visible noise textures. RFA, based on the functional threat model, employs a mapping function to alter lighting intensity from various directions while preserving crucial facial attributes such as identity and expression. This dual-pronged approach significantly expands the attack space, facilitating the discovery of adversarial perturbations with improved transferability and reduced visibility. A theoretical justification for this expansion is provided in the appendix's analysis section. To overcome the challenges of obtaining surrogate FS models and address the limited transferability due to model variations, AIR targets the identity extraction module in FS models. It generates AEs by leveraging an ensemble of face recognition models, thereby enhancing robustness and effectiveness across different FS systems.

\begin{figure}[t]
	\centering
	\includegraphics[width=\columnwidth]{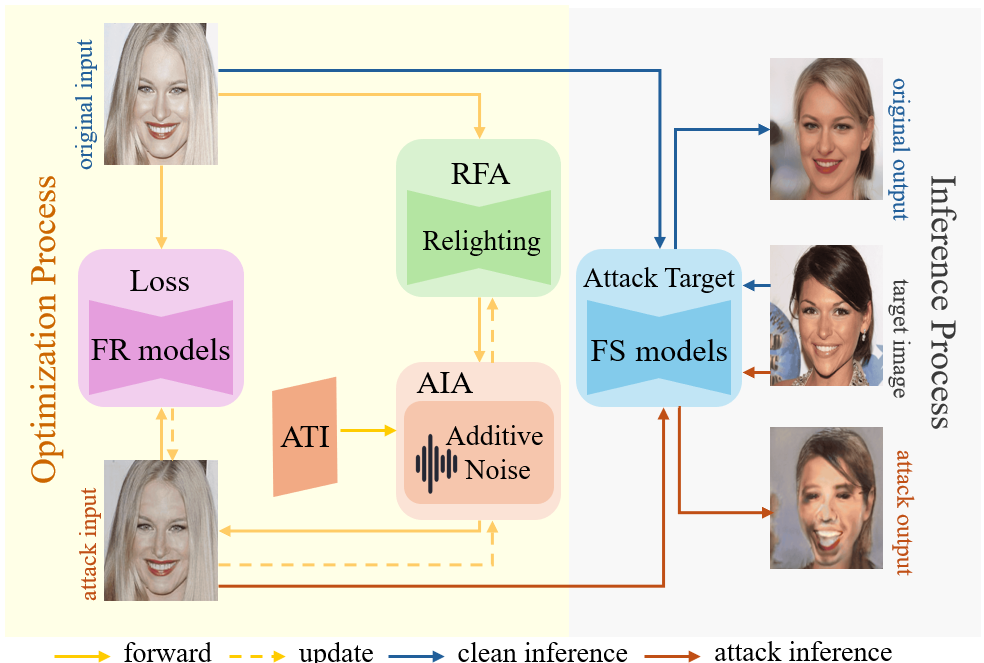}   
	\caption{The schematic diagram of AIR.}
	\label{Fig.Schematic} 
\vspace{-0.6cm}
\end{figure}


\subsection{Additive Identity Attack (AIA)}
Let $\boldsymbol{x}=(x_1,\cdots,x_{n})$ be a source image and $x_j$ be the $j^{th}$ pixel of a source image. The additive identity attack perturbs each $x_j$ by $\delta_j$. The perturbed image can be denoted as: $\boldsymbol{x'}= \boldsymbol{x}+\boldsymbol{\delta}$, where $\boldsymbol{\delta}=(\delta_1,\cdots,\delta_{n})$. $\boldsymbol{\delta}$ is generated by solving the following optimization problem:
\begin{equation}
\label{eqn:AIA_objective_function}
\centering
    {                 
		\begin{aligned}
		\mathop{\min}_{\boldsymbol{\delta}} &\sum_{i=1}^{m} \mathrm{cos} (\boldsymbol{f}_{ID_i}(\boldsymbol{x}),\boldsymbol{f}_{ID_i}
		(\boldsymbol{x}+\boldsymbol{\delta})) \\
		&s.t. \quad \| \boldsymbol{\delta} \|_\infty \leq \epsilon ,
	\end{aligned}
	}
\end{equation}
where $\boldsymbol{f}_{ID_i}(\boldsymbol{x})$ is a feature vector from the $i^{th}$ face recognition network describing the identity of $\boldsymbol{x}$ and $m$ is the total number of face recognition networks used in the ensemble. $\mathrm{cos}(\cdot,\cdot)$ denotes the cosine similarity of two vectors. 
The $l_\infty$ norm of the perturbation vector is constrained to some small value, $\epsilon$.

\subsection{Adaptive Translation-Invariant Operation}
To improve the transferability of AIA, we propose an Adaptive Translation-Invariant operation. \footnote{Technically, ATI is not translation-invariant. We keep the term translation-invariant in its name as it is modified from TI.} This novel approach builds upon the Translation-Invariant method (TI) \cite{dong2019evading}, a widely used technique for improving attack transferability. TI applies convolution to the backpropagation gradient of input with a predefined kernel, approximating the optimization over an ensemble of translated inputs. TI operations can be denoted as $\boldsymbol{W}* \nabla_{\boldsymbol{x}} J,$
where $J$ is an attack objective function and $\boldsymbol{W}$ is a predefined kernel. $\boldsymbol{W}$ could be a uniform, linear, or Gaussian kernel. While TI effectively improves transferability, it often introduces noticeable textures in the resulting images as shown in Fig. \ref{fig.ti}. 



During convolution, the computation on each sliding window is exactly the same, resulting in the same adversarial textures effect everywhere in the image, while varying in visibility depending on the image region. Smooth areas, characterized by small absolute gradient values, tend to render these textures more noticeable. To mitigate this issue, we introduce an adaptive kernel $\boldsymbol{W}_{Ada}$, which dynamically adjusts its values based on local image gradients. In areas with high gradient (rough regions), $\boldsymbol{W}_{Ada}$ approaches a uniform kernel, maintaining high transferability while allowing for more texture. Conversely, in smooth regions, the weights of $\boldsymbol{W}_{Ada}$ decay rapidly from the center to the edge of the kernel matrix, which introduces less noise textures. Specifically, for the element in the $i^{th}$ row and $j^{th}$ column of $\nabla_{\boldsymbol{x}} J$, the corresponding adaptive kernel matrix is denoted as $\boldsymbol{W}_{ij}$ of size $(2K+1)\times(2K+1)$. The value of an entry whose coordinate is $(r,c)$ in $\boldsymbol{W}_{ij}$ is defined as:
\begin{equation}
	{
		\begin{aligned}
		\widetilde w_{ij}(r,c)=&\frac{|LG_{ij}|\cdot D_{Che}((r,c),(0,0))+1}{ D_{Che}((r,c),(0,0))+1} \\
		& -K \le r,c <K \, ,
		\end{aligned}
		}
\end{equation}
where $D_{Che}((r,c),(0,0))$ is the Chebyshev distance between point $(r,c)$ and point $(0,0)$. $\boldsymbol{LG}$ is the normalized image gradient computed by 
\begin{equation*}
 {
 \boldsymbol{LG} = \frac{\boldsymbol{L}*\boldsymbol{x}}{\mathrm{max}(|\boldsymbol{L}*\boldsymbol{x}|)} \, ,    
 }
\end{equation*}
 where $\mathrm{max}(\cdot)$ returns the maximum of a matrix and $\boldsymbol{L}$ is a kernel modified from Laplacian filter as shown in Fig. \ref{fig:filter} in the appendix. $LG_{ij}$ is the element in the $i^{th}$ row and the $j^{th}$ column of $\boldsymbol{LG}$. The final weight of $\boldsymbol{W}_{ij}$ is normalized by 
 \begin{equation*}
  {
	w_{ij}(r,c)=\frac{\widetilde w_{ij}(r,c)}{\sum_{r,c} \widetilde w_{i,j}(r,c)}.    
	} 
 \end{equation*}
 
The objective functions of AIA with and without ATI are the same, but for AIA with ATI, $\nabla_{\boldsymbol{x}} J$ will be further processed by ATI in each iteration. It can be written as:
\begin{equation}
\centering
    {
		\begin{aligned}
		g_{\boldsymbol{x}} = \boldsymbol{W}_{Ada}\widetilde{*} \,   \nabla_{\boldsymbol{x}} J ,
		\end{aligned}
	}
\end{equation}
where $\boldsymbol{g}_{\boldsymbol{x}}$ is the processed gradient to update adversarial noise. $\widetilde{*}$ is an adaptive convolution that operates like regular convolution but with a variable kernel $\boldsymbol{W}_{Ada}$. The kernel $\boldsymbol{W}_{Ada}$ changes based on the convolution center, such that $\boldsymbol{W}_{Ada}$ is equals to $\boldsymbol{W}_{ij}$ when the convolution center is at the $i^{th}$ row and the $j^{th}$ column of $\boldsymbol{x}$.

\subsection{Relighting Functional Attack (RFA)}

While the AIA with ATI operation improves transferability, its effectiveness is constrained by the magnitude $\epsilon$, which must remain small to ensure imperceptibility. To further enhance the attack without increasing $\epsilon$, we introduce the relighting functional attack (RFA). In RFA, we leverage the state-of-the-art relighting model DPR \cite{zhou2019deep} as the perturbation function which excels at reilluminating high-resolution portrait images without introducing perceptible artifacts, allowing us to create natural-looking perturbations that are effective yet visually subtle. 
To generate AEs in RFA, we optimize a nine-dimensional vector ($\boldsymbol{V}_{SH}$) that contains the coefficients of the first nine Spherical Harmonic (SH) bases controlling different lighting conditions. 
As the lowest order of SH bases, $V_{SH_1}$ and $V_{SH_2}$ control the overall lighting intensity across the image. In RFA, these elements are fixed to $1$ and $0$, respectively to avoid poor illumination. 
To achieve natural reillumination, a penalty term $p(\cdot)$ is applied to $\boldsymbol{V}_{SH}$ when any of the remaining seven elements exceed a predefined interval ${\mathcal V}_{SH}$, which is defined by the mean $\mu_i$ and the standard deviation $\sigma_i$ of $V_{SH_i}$ for each dimension $i$ in training data. 
For the $i^{th}$ dimension, if $V_{SH_i} \notin {\mathcal V}_{SH_i}$, the penalty term, $p_i(V_{SH_i})= 1-e^{-\frac{1}{2}(\frac{V_{SH_i}-\mu_i}{\sigma_i})^2} \ $, is applied during optimization. AIR applies RFA and AIA with ATI sequentially, where the relighting parameters $\boldsymbol{V}_{SH}$ and $\boldsymbol{\delta}$ are optimized by the following objective function:
\begin{equation}
	{\begin{aligned}
		\mathop{\min}_{\ \boldsymbol{\delta}, \boldsymbol{V}_{SH}}  &\sum_{i=1}^{m}
		\mathrm{cos}(\boldsymbol{f}_{ID_i}(\boldsymbol{x}), 
		\boldsymbol{f}_{ID_i}({\boldsymbol{f}_{RL}(\boldsymbol{x},\boldsymbol{V}_{SH})}+\boldsymbol{\delta})) 
		+\\ 
		&\beta \sum_{j=3}^{9}p_j(V_{SH_j})\mathbb{1}_{{\mathcal V}^C_{SH_j}}\\  
		&s.t. \quad \|  \boldsymbol{\delta} \|_\infty \leq \epsilon  \ ,\ V_{SH_1}=1 , \ V_{SH_2}=0.
		\label{eq1}
	\end{aligned}}
\end{equation}
where $\boldsymbol{f}_{RL}(\cdot, \cdot)$ is a relighting perturbation function and $\beta$ is a parameter controlling the contribution of the penalty term. ${\mathcal V}^C_{SH_j}$ is the complement of ${\mathcal V}_{SH_j}$ and $\mathbb{1}_{{\mathcal V}^C_{SH_j}}$ is the indicator function of ${\mathcal V}^C_{SH_j}$.




\section{Experimental Results}
\label{sec:Experimental_Results}

In this section, we evaluate AIR on three publicly available subject-agnostic GAN-based FS models, including FaceShifter \cite{li2020advancing}, SimSwap \cite{chen2020simswap}, and MegaGAN \cite{zhu2021one}, as well as two diffusion-based models, including Diffface \cite{kim2022diffface} and Diffswap \cite{zhao2023diffswap}.  The CelebA-HQ \cite{karras2017progressive} dataset is selected for the evaluation as all the target models can generate reasonably good swapping results on it. 1000 images with different identities are randomly sampled from CelebA-HQ and taken as the source images. All the attacks use the same 1000 source images to generate the corresponding AEs. As face recognition networks serve as surrogate models, AIR does not require target images during training. In testing, another 1000 images randomly sampled from CelebA-HQ are used as target images. 

\subsection{Metrics}

To measure the identity change after attacks, attack success rate, cosine similarity, retrieval rank, and AWS Similarity are utilized. An Arcface \cite{deng2018arcface} is employed to compute the identity vectors of source images ($\boldsymbol{z}_s$), original outputs ($\boldsymbol{z}_o$) and attack outputs ($\boldsymbol{z}_{ao}$).\footnote{The Arcface utilizes a ResNet 2060 backbone trained on the MS1MV3 dataset. Note that different face recognition networks are used for evaluation, training, and as identity extractors in target FS models.} The difference between $\mathrm{cos(\boldsymbol{z}_s,\boldsymbol{z}_o)}$ and $\mathrm{cos(\boldsymbol{z}_s,\boldsymbol{z}_{ao})}$ indicates identity change. We also present the retrieval results using the original and attack outputs as searching images respectively. The 1000 source images used for the attack consist of the retrieval pool sorted by the cosine similarity between their identity vectors and the searching image's identity vector. Retrieval rank (RR) indicates the corresponding source image's rank in each retrieval. Attack success rate (ASR) is the number of successful attacks (RR $\neq 0$) divided by the total number of retrievals \cite{dong2023restricted}. Note that the rank number varies from 0 to 999.
The increase in $\mathrm{RR}$ and $\mathrm{ASR}$ from the original output to the attack output signifies the efficacy of AIR. Additionally, we include AWS Similarity (AWSS), a commercial facial recognition service as a metric. The value of AWSS varies between 0 to 100. Similarly, we measure the drop in the AWSS between the original source face input and the swapped output before and after attacks. Notably, mechanism details of AWS Rekognition Service are undisclosed, rendering it a practical black-box setting for the proposed attack.

To measure the content change before and after attacks, LPIPS and EMSE are employed. We compare the LPIPS between the original output and the attack output. The LPIPS between the original output and its JPEG-50 image is used as a reference. EMSE estimates the face contour changes of swapped results before and after the attack. It is calculated as the mean square error between edges detected in facial regions of original and attack outputs using the default parameters of the Canny edge detector from Kornia \cite{eriba2019kornia}.

\begin{table}[t]
    \caption{The attack results of the eight baselines and AIR on target models. The results of ORI are from the original output without attack. The LPIPS in the ORI row is the LPIPS between the original output and its JPEG-50 image for reference. The arrow direction indicates more effective attack.
    } 

        \centering

        \setlength{\tabcolsep}{0.30mm}
        
            \centering
            \begin{tabular}{llcccccc}
        \toprule
      & Method & $\mathrm{ASR_{}}$   $\uparrow$   & COS    $\downarrow$        & $\mathrm{RR_{}}$   $\uparrow$            & LPIPS $\uparrow$ & EMSE $\uparrow$ & AWSS $\downarrow$\\
    \midrule
    \multirow{10}{*}{\rotatebox[origin=l]{90}{MegaGAN}} & ORI    & 0.016  & 0.452  & 0.018     & 0.001  & 0.000  & 69.653 \\
    & PG     &0.021 &   0.441 	&0.080 	 	&0.003 	&0.047    &   64.533    \\
    & DA  & 0.375  & 0.287  & 12.523   & 0.007  & 0.058 & 36.888 \\ 
    & TAA  & 0.018 & 0.449  & 0.020    & 0.001  & 0.018 & 67.734  \\
    & LSA    & 0.096  & 0.395  & 4.079   & 0.003  & 0.044 & 60.933\\
    & SAA   & 0.105  & 0.388  & 1.506    & 0.003  & 0.044 & 58.576  \\
    & CMUA  & 0.157 & 0.361  & 14.957    & 0.008  & 0.056 & 53.767  \\
    & TCA-GAN &  0.017 & 0.449  & 1.069    & 0.001  &  0.028  &  59.924 \\
    & TAFIM & 0.013 &  0.454& 0.013&   0.000& 0.015&  59.639\\
    & \textbf{AIR }  & \textbf{0.796} & \textbf{0.168} & \textbf{96.294}  & \textbf{0.009} & \textbf{0.059} & \textbf{13.689}  \\ 
    
    \midrule

    \multirow{10}{*}{\rotatebox[origin=l]{90}{DiffSwap}} &  ORI    & 0.448   & 0.294     & 1.369         & 0.000 & 0.000  & 40.505     \\
    &  PG     & 0.510 &   0.281 	& 6.371 	 	& 0.011 	& 0.056  &  42.824     \\
    & DA  & 0.794 & 0.197    & 39.091        & 0.012    & 0.056 &  28.013    \\ 
    & TAA  & 0.545 & 0.278    & 5.989        & 0.011    & 0.056 & 43.713    \\
    & LSA   & 0.671  & 0.239    & 17.690       & 0.011    & 0.056 &  36.583    \\
    & SAA   & 0.706 & 0.239    & 16.649        & 0.011    & 0.056 & 37.804    \\
    & CMUA   & 0.689 & 0.235    & 20.713       & 0.012    & 0.056  & 32.092    \\
    & TCA-GAN & 0.564  & 0.276    &  5.968       &  0.011    & 0.056  & 45.079    \\
    & TAFIM  &  0.538 &  0.281    &  5.561       &   0.011    & 0.056 & 46.445   \\
    & \textbf{AIR } & \textbf{0.970    }  & \textbf{0.101} & \textbf{193.520}  & \textbf{0.013    } & \textbf{0.056    } & \textbf{7.354}    \\ 

    \bottomrule
    \end{tabular}%
        
        \label{tb:compare}
        \vspace{-0.5cm}
    \end{table}

\subsection{Comparison with Baselines}
In this section, we compare AIR with eight state-of-the-art baselines in a black-box setting: TAFIM \cite{aneja2022tafim}, PG \cite{huang2021initiative}, Disrupt Attack (DA) \cite{ruiz2020disrupting}, CMUA \cite{huang2021cmua}, TAA, SA, LSA \cite{dong2021visually}, and TCA-GAN \cite{dong2023restricted}. 
For DA and CMUA, we switch their target models to FS networks while maintaining other settings. 
In baseline attacks, another 1000 images are randomly selected from CelebA-HQ as target images during training of AEs. For AIR, eight open-source pre-trained face recognition models are employed as surrogate models. \footnote{They are composed of Arcfaces with four different backbones (ResNet 18, ResNet 34, ResNet 50, and ResNet 100) trained under two datasets (MS1MV3 and Glint360K) separately.} $\beta$ is set to 1 and $\epsilon$ is set to 0.02 for AIR. For baselines, we adhere to their original settings and employ $\epsilon= 0.05$ to ensure their effectiveness. ATI is employed in the baselines to further improve their transferability. More details about setting of baselines are included in the appendix.

The brief results are given in Table \ref{tb:compare}, and the complete results are included in Table \ref{tb:complete_compare_megagan}, \ref{tb:complete_compare_simswap}, \ref{tb:complete_compare_faceshifter}, \ref{tb:complete_compare_diffswap}, and \ref{tb:complete_compare_diffface} of the appendix. 
Fig. \ref{Fig.result_da_air} shows some attack examples against MegaGAN from AIR and DA, which are the best two methods in Table \ref{tb:compare}. More attack examples of AIR against these FS models are given in the appendix. Note that different similarity measurement can lead to varying degrees of change in AEs. For example, works like \cite{laidlaw2019functional} and Advmakeup \cite{yin2021adv} can significantly alter original images in terms of $L_{\infty}$, while still appearing natural to human observers. Similarly, in this study, the difference in illumination between the AE and the original image is considered acceptable within our framework.

Table \ref{tb:compare} shows that AIR performs the best on all the FS models on all the metrics. Even with ATI, all baselines exhibit ASR values below 0.2 on GAN-based FS models, except for DA against MegaGAN, whose maximum ASR is 0.375. 
AWSS values for the baselines range from 32.092 to 91.370, except for DA against DiffSwap, which drops to 28.013. Notably, generator-based attack PG, TAFIM, and TCA-GAN also fail to transfer across FS networks, likely due to overfitting on surrogate models and specific training datasets. These findings suggest the inability of baseline methods to safeguard images against the five FS models in a black-box setting. In contrast, AIR's ASR and AWSS values across the five target models range from 0.699 to 0.970 and 7.080 to 14.877, respectively, indicating effective distortion of image identity output from FS networks. 
The findings presented in Table \ref{tb:compare} underscore AIR's remarkable efficacy in the black-box setting across the state-of-the-art FS models of different architectures. Except for the above-mentioned baselines, we further compare AIR with attacks modified from TI and TAIG \cite{huang2022transferable}, two state-of-the-art transferable attacks against image classification in the appendix, which demonstrates the superiority of AIR.


We further evaluate the robustness of the baseline methods against various basic image operations, including {JPEG compression}, {resizing}, and {rotation}. Our results, detailed in the appendix, indicate that these pre-processing operations have only a minimal impact on the performance of all methods when using ATI. Notably, even with these operations, AIR consistently outperforms the other methods. Due to space constraints, additional details are provided in the appendix.


\begin{figure} 
        \centering 
        \includegraphics[width=\columnwidth]{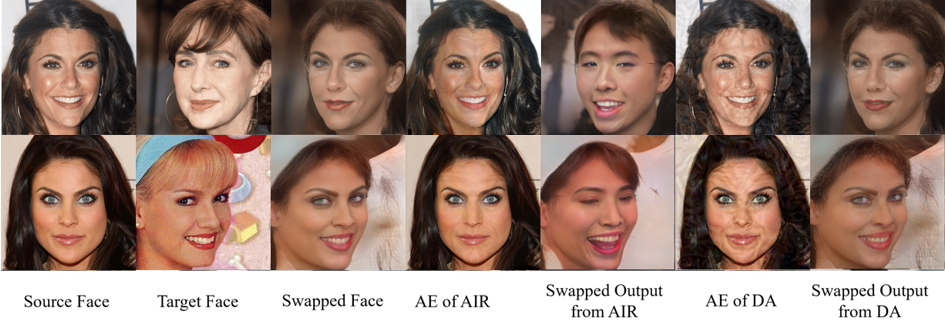} 
     \caption{Face swapping result of MegaGAN before and after AIR and DA attack.} 
        \label{Fig.result_da_air}
     \vspace{-0.3cm}
    \end{figure}

\subsection{Image Quality Evaluation}    
To further validate the image quality of AEs, we conduct a user study with approval from the Institutional Review Board. The study evaluates whether adversarial perturbations generated by the top three methods (AIR, DA, and CMUA) are noticeable and impact image quality. Participants rate the artifacts on the images on a scale of 1 to 5, where a higher score indicates more noticeable artifacts. More details of the user study are included in the appendix. We gather and compare the ratings from 72 participants, finding that 61.1\% and 74.1\% of the ratings for CMUA and DA, respectively, are 4 or 5. In contrast, only 34.5\% of AIR's ratings are 4 or 5, which are significantly lower than the other two methods. Notably, even original images have 7.4\% of ratings of 4 and 5. These results suggest that AIR produces less noticeable adversarial perturbations than DA and CMUA, with minimal impact on image quality. 
Besides subjective evaluation, we also assess the image quality objectively with metrics including CLIP-IQA \cite{wang2023exploring} and Total Variation (TV). Higher CLIP-IQA and lower TV indicate better image quality. CMUA, DA, and AIR achieve 156.3, 61.0, and \textbf{47.8} for TV and 0.37, 0.22 and \textbf{0.54} for CLIP-IQA. The results show that AIR also provides higher image quality, aligning with our user study.


\subsection{Ablation Study}

\textbf{The influence of $\beta$.} The user study indicates that modifying the illumination is generally acceptable. For users who prefer to retain the original illumination as much as possible, they can adjust the parameter $\beta$ to control illumination. Fig. \ref{fig_relighting_beta} shows AEs of AIR with different $\beta$. Experiments with $\beta=1,2$ and $10$ are detailed in Table \ref{tb:ablation}. In an extreme case, AIR without RFA, (i.e., AIA) can still outperform the baselines as shown in Table \ref{tb:compare} and Table \ref{tb:ablation}.


\begin{figure}[t]
         \centering
         \includegraphics[width=\columnwidth]{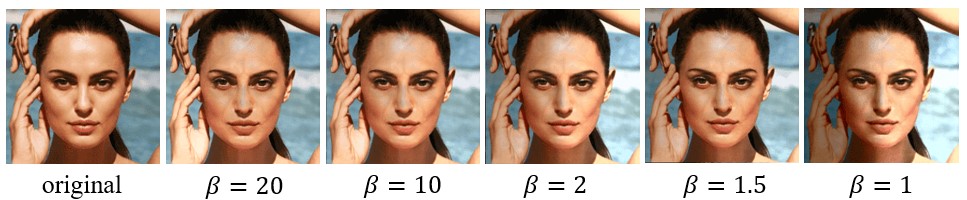}
    \caption{AEs of AIR with different $\beta$ for natural relighting.}
    \label{fig_relighting_beta}
    \vspace{-0.5cm}
\end{figure}

\begin{figure}[t]
    \centering
    \includegraphics[width=\columnwidth]{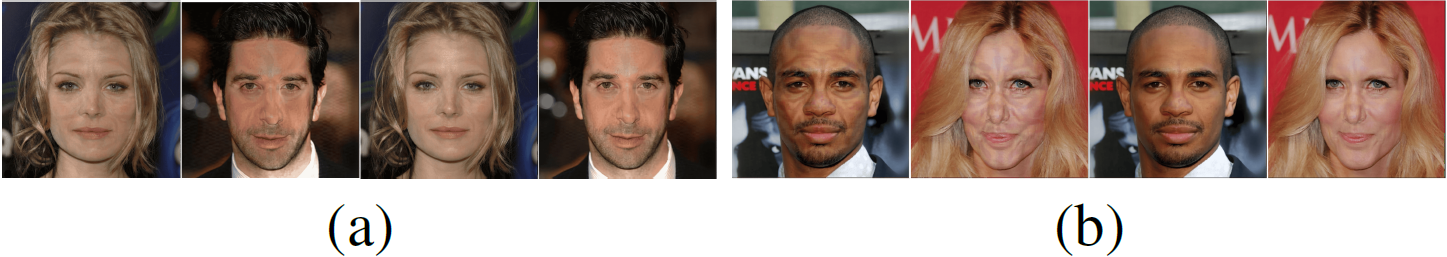}
\caption{Adversarial examples generated by AIR with TI and ATI for (a) $\epsilon=0.02$ and (b) $\epsilon=0.05$. The first two images in each set are from AIR with TI, and the last two are from AIR with ATI. (Please zoom in for a closer view of the details.)}
 \label{Fig:ATI}
 \vspace{-0.3cm}    
\end{figure}

\begin{table}[t]
    \caption{Summary of ablation study results on MegaGAN. AIR without ATI is denoted as $\mathrm{AIR_{woATI}}$, and with surrogate models including FS networks as $\mathrm{AIR_{FS}}$. $\mathrm{AIR_{woATI}}$, $\mathrm{AIR_{FS}}$ and ATI are tested under $\epsilon=0.02$.
    AIA with $\epsilon=0.02$ and $\epsilon=0.05$ are denoted as $\mathrm{AIA_{0.02}}$ and $\mathrm{AIA_{0.05}}$, respectively.}
    \centering
            \setlength{\tabcolsep}{0.4mm}
     \begin{tabular}{lcccccc}
\toprule
Method & $\mathrm{ASR_{}}$   $\uparrow$ & COS    $\downarrow$        & $\mathrm{RR_{}}$   $\uparrow$             & LPIPS $\uparrow$ & EMSE $\uparrow$ & AWSS $\downarrow$\\
\midrule
ORI  & 0.016  & 0.452  & 0.018  & 0.001 & 0.000 & 69.653   \\
$\mathrm{AIA_{0.02}}$  & 0.444 & 0.286 & 9.606    & 0.004 & 0.051 & 35.108  \\
$\mathrm{AIA_{0.05}}$   & 0.628 & 0.231 & 30.508  & 0.006 & 0.054 & 24.204   \\
RFA         & 0.536  & 0.273 & 62.596   & 0.009 & 0.048 & 34.736  \\
$\mathrm{AIR_{woATI}}$   & 0.163 & 0.371 & 11.991   & 0.005 & 0.045 & 55.641   \\
$\mathrm{AIR_{FS}}$    & 0.502 & 0.260 & 24.470   & 0.006 & 0.054 & 32.313   \\
$\mathrm{AIR_{\beta=10}}$        & 0.665  &0.215 & 40.400   & 0.007 & 0.056 & 21.397 \\
$\mathrm{AIR_{\beta=2}}$      & 0.765  & 0.185 & 67.825   & 0.008 & 0.058 & 15.983 \\
\textbf{$\mathrm{AIR_{\beta=1}}$}  & \textbf{0.796}& \textbf{0.168} & \textbf{96.294}  & \textbf{0.009} & \textbf{0.059} & \textbf{13.689}  \\

\bottomrule
\end{tabular}%

    \label{tb:ablation}
    \vspace{-0.5cm}
\end{table}

\textbf{Comparison between AIA and RFA with AIR.} In this evaluation, RFA and AIA are performed separately with the setting used in AIR. In AIA, we set $\epsilon$ to 0.02 and 0.05 to investigate the influence of the additive noise. In Table \ref{tb:ablation}, it shows that AIR performs much better than RFA and AIA. Increasing $\epsilon$ from 0.02 to 0.05 significantly improves the attack of AIA but with a more perceptible noise level. However, it is still weaker than AIR. Some AEs generated by AIA under different $\epsilon$ are given in the appendix.

\textbf{The effectiveness of ATI}. In this evaluation, we compare the attack results of AIR with and without ATI under $\epsilon=0.02$. As shown in Table \ref{tb:ablation}, ATI effectively improves the transferability of AIR. Interestingly, RFA outperforms AIR without ATI because RFA's attack space is smoother; AIR without ATI is more prone to local optima, leading to reduced performance. We also compare the AEs generated by AIR with TI and ATI. As shown in Fig. \ref{Fig:ATI}, AEs with TI exhibit noticeable noise textures, particularly at $\epsilon=0.05$, while ATI significantly weakens these textures.  
Table \ref{tb:ablation} and Fig. \ref{Fig:ATI} demonstrate that ATI can help improve the transferability of the attacks while avoiding obvious noise textures caused by the original TI.

\textbf{Incorporating face swapping networks into surrogate models}. In this study, we insert FS networks in the surrogate models. A source image and a training target image are fed into the FS models first. Then the swapped result is taken as the input of the face recognition networks used in AIR. The black-box attack is performed by minimizing the cosine similarity between the original swapped output and the attacked swapped output. The same set of images used in the baseline methods is employed in this evaluation. Table \ref{tb:ablation} shows that inserting FS models in the surrogate model does not improve the transferability to other FS models.

\textbf{More experiments are provided in the appendix\footnote{The appendix is available in the arXiv version of this paper.}, including ablation studies on different face recognition models, different datasets, etc.}
\section{Conclusion}
We introduce AIR, which bypasses the limitations of using FS models in training. AIR integrates AIA and RFA, both targeting the identity extraction process in subject-agnostic FS models. The proposed ATI further enhances AIA's transferability while maximizing imperceptibility. By employing these dual strategies, AIR expands the attack space, achieving a higher success rate in transferable attack on FS models without introducing noticeable noise.

\section*{Acknowledgments}
This work is conducted at ROSE @ NTU, Interdisciplinary Graduate Programme, Nanyang Technological University, Singapore, supported by NTU Internal Funding - Accelerating Creativity and Excellence (NTU-ACE2020-03).

\bibliographystyle{IEEEbib}
\bibliography{icme2025_DF_defense}

\clearpage
\appendix
\label{sec:SM}

\subsection{Analysis of a Combined Threat Model}
\label{sec:Analysis}

In AIR, the search for AEs is conducted under a combined threat model that integrates both the additive and functional threat models. This combination results in an effective expansion of the attack space, thereby enabling the generation of adversarial perturbations that exhibit high transferability and low visibility. In this subsection, we provide a theoretical analysis for this design.

Let the entire image space be $\mathcal{Z} = \{\boldsymbol{X}=(x_1,\ldots,x_n)| \, \forall x_i \in [0,1]\}$.
Let $\boldsymbol{X}_0$ be an image in $\mathcal{Z}$ and the additive noise space of $\boldsymbol{X}_0$ be $  \mathcal{N}_{L_p,\epsilon}(\boldsymbol{X}_0)=\{\boldsymbol{\delta} | \|\boldsymbol{\delta}\|_p \leq \epsilon, \boldsymbol{X}_0+\boldsymbol{\delta} \in \mathcal{Z}\}.$
For the sake of convenience, let $\mathcal{Z}_{int}$ be the interior of $\mathcal{Z}$ and $\partial \mathcal{Z}$ be the boundary of $\mathcal{Z}$. The same notations are applied to other sets in this section. We also denote the functional attack space of $\boldsymbol{X}_0$ as $\mathcal{S}(\boldsymbol{X}_0)=\{\boldsymbol{X}|\boldsymbol{X}=\boldsymbol{f}_{a} (\boldsymbol{X}_0,\boldsymbol{\theta}),\boldsymbol{X} \in \mathcal{Z}, \boldsymbol{\theta} \in \psi\},$
where $\boldsymbol{f}_{a}$ is a function employed in the functional attack, $\boldsymbol{\theta}$ is its parameter and $\psi$ is the set of valid parameters. In RFA, $\boldsymbol{f}_a$ is the relighting network and  $\boldsymbol{\theta}$ is $\boldsymbol{V}_{SH}$. It is assumed that $\boldsymbol{X}_0 \in \mathcal S (\boldsymbol{X}_0)$. In the case of RFA, it means no relighting is applied on $\boldsymbol{X}_0$. Moreover, we define the additive attack space of $\boldsymbol{X}_0$ as $\mathcal{B}_{L_p,\epsilon}( \boldsymbol{X}_0)=\{\boldsymbol{X}|\boldsymbol{X}=\boldsymbol{X}_0+\boldsymbol{\delta}, \boldsymbol{X} \in \mathcal{Z}, \boldsymbol{\delta} \in \mathcal{N}_{L_p,\epsilon}(\boldsymbol{X}_0)\},$ 
and the noise space of the functional attack as $ \mathcal{N}_{S}(\boldsymbol{X}_0)=\{\boldsymbol{\rho}|\boldsymbol{\rho} = \boldsymbol{X}-\boldsymbol{X}_0, \boldsymbol{X} \in \mathcal{S}(\boldsymbol{X}_0)\}.$
Note that $\boldsymbol{\rho}$ is the difference between output of $\boldsymbol{f}_{a}$ and $\boldsymbol{X}_0$. Similarly, we denote $ \Omega(\boldsymbol{X}_0)=\{\boldsymbol{W}|\boldsymbol{W}=\boldsymbol{\delta}+\boldsymbol{\rho}+\boldsymbol{X}_0, \boldsymbol{\delta} \in \mathcal{N}_{L_p,\epsilon}( \boldsymbol{X}_0), \boldsymbol{\rho} \in \mathcal{N}_\mathcal{S}( \boldsymbol{X}_0), \boldsymbol{W}\in \mathcal{Z}\}$
as the combined attack space. 
\begin{mydef}
     Given a subset $\Phi$ of $\mathcal{Z}$, if $\Phi$ is a closed set containing $\boldsymbol{X}_0$ and there exist a unit vector $\boldsymbol{v}$ and a point $\boldsymbol{X} \in \partial \Phi$ such that $\frac{(\boldsymbol{X}-\boldsymbol{X}_0)}{\| \boldsymbol{X}-\boldsymbol{X}_0\|}= \boldsymbol{v}$ and $\boldsymbol{X}+\alpha \boldsymbol{v} \in \Phi^C, \forall \  \alpha > 0$ when $\boldsymbol{X}+\alpha \boldsymbol{v} \in \mathcal{Z}$, then $\Phi$ is called a some-way-out set of $\boldsymbol{X}_0$.
\end{mydef}
\begin{mydef}
    Given a subset $\Phi$ of $\mathcal{Z}$, if $\Phi$ is a closed set containing $\boldsymbol{X}_0$ and given any unit vector $ \boldsymbol{v}$, there exists a point $\boldsymbol{X} \in \partial \Phi$  such that $\frac{(\boldsymbol{X}-\boldsymbol{X}_0)}{\| \boldsymbol{X}-\boldsymbol{X}_0\|}= \boldsymbol{v}$ and $\boldsymbol{X}+\alpha \boldsymbol{v} \in \Phi^C, \forall \  \alpha > 0$ when $\boldsymbol{X}+\alpha \boldsymbol{v} \in \mathcal{Z}$, then  $\Phi$ is called an any-way-out set of $\boldsymbol{X}_0$.
\end{mydef} 

\begin{theorem}
\label{theorem1}
If $\boldsymbol{X}_0 \in \mathcal{Z}_{int}$, $\epsilon$ is small enough such that $\mathcal{B}_{L_p,\epsilon} (\boldsymbol{X}_0) \in \mathcal{Z}_{int}$ and $\mathcal{S}(\boldsymbol{X}_0)$ is a some-way-out set of $\boldsymbol{X}_0$, then $\mathcal{B}_{L_p,\epsilon}(\boldsymbol{X}_0)\cup\mathcal{S}(\boldsymbol{X}_0)$ is a proper subset of $\Omega(\boldsymbol{X}_0)$ 
\end{theorem}

\begin{proof}
  Selecting $\boldsymbol{\rho}=\boldsymbol{0}$ and $\boldsymbol{W}= \boldsymbol{\delta}+\boldsymbol{X}_0$. Thus, $\mathcal{B}_{L_p,\epsilon} (\boldsymbol{X}_0)$ is a subset of $\Omega(\boldsymbol{X}_0)$. Similarly, selecting 
  $\boldsymbol{\delta}=\boldsymbol{0}, \, \boldsymbol{W}=\boldsymbol{\rho}+\boldsymbol{X}_0$. Thus, $\mathcal{S}(\boldsymbol{X}_0)$ is a subset of $\Omega(\boldsymbol{X}_0)$. By the definition of
$\mathcal{N}_{L_p,\epsilon}( \boldsymbol{X}_0)$ and $\mathcal{B}_{L_p,\epsilon}( \boldsymbol{X}_0) \subset \mathcal{Z}_{int}$, $\mathcal{B}_{L_p,\epsilon}$ is an any-way-out set. Since $\mathcal{S}(\boldsymbol{X}_0)$ is a some-way-out set of $\boldsymbol{X}_0$, there exists $\boldsymbol{Y} \in \partial \mathcal{S}(\boldsymbol{X}_0)$ and a unit vector $\boldsymbol{v}$ defined as:
\begin{equation*}
    \boldsymbol{v} = \frac{(\boldsymbol{Y}-\boldsymbol{X}_0)}{\| \boldsymbol{Y}-\boldsymbol{X}_0\|}, \; s.t. \, \boldsymbol{Y}+ \beta \boldsymbol{v} \in \mathcal{S}(\boldsymbol{X}_0)^C, \; \forall \beta > 0.
\end{equation*}
Since $\mathcal{B}_{L_p,\epsilon}( \boldsymbol{X}_0)$ is an any-way-out set, there exists a point 
\begin{equation*}
    \boldsymbol{X} \in \partial \mathcal{B}_{L_p,\epsilon}( \boldsymbol{X}_0), \; s.t. \frac{(\boldsymbol{X}-\boldsymbol{X}_0)}{\| \boldsymbol{X}-\boldsymbol{X}_0\|}= \boldsymbol{v},
\end{equation*}
and $\boldsymbol{X}+\alpha \boldsymbol{v} \in \mathcal{B}_{L_p,\epsilon}( \boldsymbol{X}_0)^C, \forall \alpha > 0$. 
 Let $\delta=\boldsymbol{X}-\boldsymbol{X}_0$, $\rho=\boldsymbol{Y}-\boldsymbol{X}_0$ and $\boldsymbol{W}=\boldsymbol{X}_0+\boldsymbol{\delta}+\boldsymbol{\rho}$. Thus, 
 \begin{align*}
 \begin{aligned}
 \boldsymbol{W}=& \boldsymbol{X}_0+  \boldsymbol{X}- \boldsymbol{X}_0+  \boldsymbol{Y} -\boldsymbol{X}_0 
 =  \boldsymbol{X}+  \boldsymbol{Y}- \boldsymbol{X}_0 \\
 = &\boldsymbol{Y}+\| \boldsymbol{X}-\boldsymbol{X}_0\| \frac{(\boldsymbol{X}-\boldsymbol{X}_0)}{\| \boldsymbol{X}-\boldsymbol{X}_0\|} 
 = \boldsymbol{Y}+\beta \boldsymbol{v},   
 \end{aligned}
 \end{align*}
 where $\beta = \| \boldsymbol{X}-\boldsymbol{X}_0\|>0$. Thus, $\boldsymbol{W} \in \mathcal{S}(\boldsymbol{X}_0)^C$. Similarly, 
 \begin{align*}
 \begin{aligned}
 \boldsymbol{W}=& \boldsymbol{X}+  \boldsymbol{Y}- \boldsymbol{X}_0 \\
 = &  \boldsymbol{X}+\| \boldsymbol{Y}-\boldsymbol{X}_0\| \frac{(\boldsymbol{Y}-\boldsymbol{X}_0)}{\| \boldsymbol{Y}-\boldsymbol{X}_0\|} \\
 = &\boldsymbol{X}+\alpha \boldsymbol{v},
 \end{aligned}
 \end{align*}
 where $\alpha = \| \boldsymbol{Y}-\boldsymbol{X}_0\|>0$. Thus, $\boldsymbol{W} \in \mathcal{B}_{L_p,\epsilon} (\boldsymbol{X}_0)^C$. In other words, $\boldsymbol{W} \in \Omega(\boldsymbol{X}_0)$ but $\boldsymbol{W} \notin \mathcal{B}_{L_p,\epsilon} (\boldsymbol{X}_0) \cup \mathcal{S}(\boldsymbol{X}_0).$ It proves that the combined attack space $\Omega(\boldsymbol{X}_0)$ is larger than the union of the two attack spaces, $\mathcal{B}_{L_p,\epsilon} (\boldsymbol{X}_0)$ and $\mathcal{S}(\boldsymbol{X}_0)$. 
\end{proof}

According to Theorem \ref{theorem1}, there exist $\boldsymbol{W}$ such that $\boldsymbol{W} \in \Omega(\boldsymbol{X}_0)$ but $\boldsymbol{W} \notin \mathcal{B}_{L_p,\epsilon} (\boldsymbol{X}_0) \cup \mathcal{S}(\boldsymbol{X}_0)$, proving that the combined attack space $\Omega(\boldsymbol{X}_0)$ is larger than the union of the two attack spaces, $\mathcal{B}_{L_p,\epsilon} (\boldsymbol{X}_0)$ and $\mathcal{S}(\boldsymbol{X}_0)$.

\subsection{Complete Results of Comparison with Baselines}

  \begin{table}[h!]
        \centering

        \caption{The attack results of the eight baselines and AIR on \textbf{MegaGAN}. The results of ORI are from the original output without attack. The LPIPS in the ORI row is the LPIPS between the original output and its JPEG-50 image for reference. The arrow direction indicates more effective attack.
        } 

         
            \centering
            \setlength{\tabcolsep}{0.4mm}
            \begin{tabular}{lcccccc}
      \toprule
    Method  & $\mathrm{ASR_{}}$   $\uparrow$ & COS    $\downarrow$        & $\mathrm{RR_{}}$   $\uparrow$    & LPIPS $\uparrow$ & EMSE $\uparrow$ & AWSS $\downarrow$\\
    \midrule
     ORI    & 0.016  & 0.452  & 0.018     & 0.001  & 0.000  & 69.653 \\
     PG    &0.021   &   0.441 	&0.080 		&0.003 	&0.047    &   64.533    \\
    DA   & 0.375 & 0.287  & 12.523   & 0.007  & 0.058 & 36.888 \\ 
    TAA & 0.018 & 0.449  & 0.020     & 0.001  & 0.018 & 67.734  \\
    LSA   & 0.096 & 0.395  & 4.079     & 0.003  & 0.044 & 60.933\\
    SAA   & 0.105  & 0.388  & 1.506    & 0.003  & 0.044 & 58.576  \\
    CMUA & 0.157 & 0.361  & 14.957    & 0.008  & 0.056 & 53.767  \\
    TCA-GAN &  0.017 & 0.449  & 1.069    & 0.001  &  0.028  &  59.924 \\
    TAFIM &  0.013 & 0.454& 0.013&  0.000& 0.015&  59.639\\
    \textbf{AIR } & \textbf{0.796} & \textbf{0.168} & \textbf{96.294}   & \textbf{0.009} & \textbf{0.059} & \textbf{13.689}  \\ 
  
      \bottomrule

    \end{tabular}%
            
      \label{tb:complete_compare_megagan}
    \end{table}

    \begin{table}[h!]
          \centering
          \caption{The attack results of the eight baselines and AIR on \textbf{SimSwap}. The results of ORI are from the original output without attack. The LPIPS in the ORI row is the LPIPS between the original output and its JPEG-50 image for reference. The arrow direction indicates more effective attack.
          } 

           
              \centering
              \setlength{\tabcolsep}{0.4mm}
              \begin{tabular}{lcccccc}
        \toprule
      Method  & $\mathrm{ASR_{}}$   $\uparrow$ & COS    $\downarrow$        & $\mathrm{RR_{}}$   $\uparrow$    & LPIPS $\uparrow$ & EMSE $\uparrow$ & AWSS $\downarrow$\\
      \midrule

       ORI & 0.007  & 0.492 & 0.007 & 0.004 & 0.000 & 69.303  \\
        PG   & 0.016 & 0.464 & 0.042 & 0.001 & 0.019 &69.549  \\
        DA  & 0.096 & 0.400 & 0.619  & 0.005 & 0.036 & 52.764  \\
        TAA  & 0.007 & 0.489 & 0.008  & 0.001 & 0.008 & 67.957 \\
       LSA & 0.029 & 0.458 & 0.077  & 0.002 & 0.026 & 63.801  \\
        SAA  & 0.022  & 0.470 & 0.044& 0.002 & 0.026 & 65.473  \\
        CMUA  & 0.053 & 0.423 & 0.546  & 0.004 & 0.034 & 57.180 \\
        TCA-GAN &  0.010 & 0.481  & 0.013    &  0.000  &  0.010  & 68.223  \\
        TAFIM & 0.012 & 0.481  &  0.015   & 0.000  & 0.007  &  69.789  \\
      \textbf{AIR}   & \textbf{0.699}  & \textbf{0.191} & \textbf{82.384}  & \textbf{0.010} & \textbf{0.048} & \textbf{14.877} \\
     
        \bottomrule

      \end{tabular}%
              
        \label{tb:complete_compare_simswap}
      \end{table}

      \begin{table}[h!]
            \centering

            \caption{The attack results of the eight baselines and AIR on \textbf{FaceShifter}. The results of ORI are from the original output without attack. The LPIPS in the ORI row is the LPIPS between the original output and its JPEG-50 image for reference. The arrow direction indicates more effective attack.
            } 

             
                \centering
                \setlength{\tabcolsep}{0.4mm}
                \begin{tabular}{lcccccc}
          \toprule
        Method  & $\mathrm{ASR_{}}$   $\uparrow$ & COS    $\downarrow$        & $\mathrm{RR_{}}$   $\uparrow$    & LPIPS $\uparrow$ & EMSE $\uparrow$ & AWSS $\downarrow$\\
        \midrule
    
          ORI  & 0.009 & 0.458 & 0.013  & 0.003 & 0.000 & 70.552 \\
          PG  & 0.006 & 0.447 & 0.006  & 0.001 & 0.020 & 70.098  \\
          DA  & 0.157 & 0.348 & 2.759 & 0.006 & 0.042 & 49.164 \\
          TAA  & 0.009 & 0.454 & 0.012 & 0.000 & 0.007 & 69.797  \\
          LSA  & 0.067 & 0.378 & 0.488 & 0.005 & 0.040 & 55.718  \\
          SAA & 0.037 & 0.431 & 0.058 & 0.001 & 0.020 & 65.573  \\
          CMUA & 0.066 & 0.393 & 0.437 & 0.005 & 0.039 & 58.323 \\
          TCA-GAN &  0.010  & 0.458  &  0.011  &  0.000  & 0.014  & 69.951  \\
          TAFIM &   0.010 & 0.459  &  0.012   &  0.000
          &  0.005  &  67.901  \\
        \textbf{AIR} & \textbf{0.811}  & \textbf{0.159} & \textbf{103.810}  & \textbf{0.010} & \textbf{0.051} & \textbf{11.267}\\

          \bottomrule

        \end{tabular}%
                
          \label{tb:complete_compare_faceshifter}
        \end{table}

        \begin{table}[h!]
              \centering
              \caption{The attack results of the eight baselines and AIR on \textbf{DiffSwap}. The results of ORI are from the original output without attack. The LPIPS in the ORI row is the LPIPS between the original output and its JPEG-50 image for reference. The arrow direction indicates more effective attack.
              } 

               
                  \centering
                  \setlength{\tabcolsep}{0.4mm}
                  \begin{tabular}{lcccccc}
            \toprule
          Method  & $\mathrm{ASR_{}}$   $\uparrow$ & COS    $\downarrow$        & $\mathrm{RR_{}}$   $\uparrow$    & LPIPS $\uparrow$ & EMSE $\uparrow$ & AWSS $\downarrow$\\
          
          \midrule
             ORI    & 0.448   & 0.294     & 1.369         & 0.000 & 0.000  & 40.505     \\
             PG     & 0.510 &   0.281 	& 6.371 	 	& 0.011 	& 0.056  &  42.824     \\
            DA  & 0.794 & 0.197    & 39.091        & 0.012    & 0.056 &  28.013    \\ 
            TAA  & 0.545 & 0.278    & 5.989        & 0.011    & 0.056 & 43.713    \\
            LSA   & 0.671  & 0.239    & 17.690       & 0.011    & 0.056 &  36.583    \\
            SAA   & 0.706 & 0.239    & 16.649        & 0.011    & 0.056 & 37.804    \\
            CMUA   & 0.689 & 0.235    & 20.713       & 0.012    & 0.056  & 32.092    \\
            TCA-GAN & 0.564  & 0.276    &  5.968       &  0.011    & 0.056  & 45.079    \\
            TAFIM  &  0.538 &  0.281    &  5.561       &   0.011    & 0.056 & 46.445   \\
            \textbf{AIR } & \textbf{0.970    }  & \textbf{0.101} & \textbf{193.520}  & \textbf{0.013    } & \textbf{0.056    } & \textbf{7.354}    \\ 
         
            \bottomrule

          \end{tabular}%
                  
            \label{tb:complete_compare_diffswap}
          \end{table}

          \begin{table}[H]
                \centering
                \caption{The attack results of the eight baselines and AIR on \textbf{Diffface}. The results of ORI are from the original output without attack. The LPIPS in the ORI row is the LPIPS between the original output and its JPEG-50 image for reference. The arrow direction indicates more effective attack.
                } 
                 
                    \centering
                    \setlength{\tabcolsep}{0.4mm}
                    \begin{tabular}{lcccccc}
              \toprule
            Method  & $\mathrm{ASR_{}}$   $\uparrow$ & COS    $\downarrow$        & $\mathrm{RR_{}}$   $\uparrow$    & LPIPS $\uparrow$ & EMSE $\uparrow$ & AWSS $\downarrow$\\
            \midrule
                ORI & 0.009 & 0.458  & 0.013    & 0.003 & 0.000 & 90.523      \\
                PG  & 0.066 & 0.461      & 0.566           & 0.007      &     0.034    & 88.515      \\
                DA  & 0.402 & 0.313      & 22.738            & 0.011      & 0.036    & 91.319      \\
                TAA  & 0.019  & 0.545      & 0.101           & 0.008      & 0.045  & 89.543      \\
                LSA  & 0.072 & 0.375      & 2.994            & 0.013      & \textbf{0.047}   & 63.087      \\
                SAA & 0.144 & 0.395      & 2.367            & 0.012      & 0.047  & 70.777      \\
                CMUA & 0.076 & 0.442      & 0.713            & 0.008      & 0.035  & 75.064      \\
                TCA-GAN  &   0.020 & 0.539      &  0.237          &  0.008      & 0.045  & 89.059      \\
                TAFIM &  0.021  &  0.548      &  0.056           &  0.003      &  0.028      &91.370
                \\
              \textbf{AIR}  & \textbf{0.902    } & \textbf{0.116    } & \textbf{190.082    }  & \textbf{0.013 } & 0.038 & \textbf{7.080}  \\
          
              \bottomrule

            \end{tabular}%
                    
              \label{tb:complete_compare_diffface}
            \end{table}

\subsection{Complete Table of Ablation Study}

\begin{table}[H]

        \centering

        \caption{Summary of ablation study results on MegaGAN, SimSwap and FaceShifter. AIR with and without ATI are denoted as AIR and $\mathrm{AIR_{woATI}}$, respectively. AIR with surrogate models including FS networks is denoted as $\mathrm{AIR_{FS}}$. $\mathrm{AIR_{woATI}}$, $\mathrm{AIR_{FS}}$ and ATI are performed under $\epsilon=0.02$.
        AIA with $\epsilon=0.02$ and $\epsilon=0.05$ are denoted as $\mathrm{AIA_{0.02}}$ and $\mathrm{AIA_{0.05}}$, respectively.}

        \setlength{\tabcolsep}{0.4mm}
         \begin{tabular}{llcccccc}
	\toprule
  & Method & $\mathrm{ASR_{}}$   $\uparrow$ & COS    $\downarrow$        & $\mathrm{RR_{}}$   $\uparrow$             & LPIPS $\uparrow$ & EMSE $\uparrow$ & AWSS $\downarrow$\\
\midrule
\multirow{9}{*}{\rotatebox[origin=l]{90}{MegaGAN}} & ORI   & 0.016 & 0.452 & 0.018    & 0.001 & 0.000 & 69.653   \\
& $\mathrm{AIA_{0.02}}$   & 0.444 & 0.286 & 9.606   & 0.004 & 0.051 & 35.108  \\
& $\mathrm{AIA_{0.05}}$   & 0.628  & 0.231 & 30.508  & 0.006 & 0.054 & 24.204   \\
& RFA         & 0.536  & 0.273 & 62.596   & 0.009 & 0.048 & 34.736  \\
& $\mathrm{AIR_{woATI}}$  & 0.163  & 0.371 & 11.991   & 0.005 & 0.045 & 55.641   \\
& $\mathrm{AIR_{FS}}$    & 0.502  & 0.260 & 24.470  & 0.006 & 0.054 & 32.313   \\
& $\mathrm{AIR_{\beta=10}}$   & 0.665       &0.215 & 40.400   & 0.007 & 0.056 & 21.397 \\
& $\mathrm{AIR_{\beta=2}}$    & 0.765   & 0.185 & 67.825    & 0.008 & 0.058 & 15.983 \\
& \textbf{$\mathrm{AIR_{\beta=1}}$} & \textbf{0.796 }& \textbf{0.168} & \textbf{96.294}  & \textbf{0.009} & \textbf{0.059} & \textbf{13.689}  \\

\midrule
\multirow{9}{*}{\rotatebox[origin=l]{90}{SimSwap}} & ORI & 0.007& 0.492 & 0.007   & 0.004 & 0.000 & 69.303   \\
& $\mathrm{AIA_{0.02}}$   & 0.211    & 0.354 & 3.072    & 0.004 & 0.035 & 40.819   \\
& $\mathrm{AIA_{0.05}}$       & 0.307     & 0.313 & 7.090   & 0.006 & 0.039 & 32.802   \\
& RFA   & 0.523    & 0.273 & 77.814   & 0.010 & 0.039 & 31.130   \\
& $\mathrm{AIR_{woATI}}$& 0.127 & 0.388 & 1.858    & 0.005 & 0.034 & 49.577   \\
& $\mathrm{AIR_{FS}}$  & 0.335    & 0.298 & 12.961   & 0.007 & 0.042 & 30.886   \\
& $\mathrm{AIR_{\beta=10}}$   & 0.426    & 0.272 & 16.467   & 0.007& 0.042 & 29.753 \\
& $\mathrm{AIR_{\beta=2}}$    & 0.566    & 0.231 &  37.313  & 0.008& 0.045 & 23.789  \\
& \textbf{$\mathrm{AIR_{\beta=1}}$}  & \textbf{0.699} & \textbf{0.191} & \textbf{82.384}   & \textbf{0.010} & \textbf{0.048} & \textbf{14.877}  \\
\midrule  
\multirow{9}{*}{\rotatebox[origin=l]{90}{FaceShifter}} & ORI & 0.009& 0.458 & 0.013   & 0.003 & 0.000 & 70.552   \\
& $\mathrm{AIA_{0.02}}$    & 0.366       & 0.286 & 13.036  & 0.005 & 0.041 & 32.592   \\
& $\mathrm{AIA_{0.05}}$     & 0.588      & 0.225 & 31.252   & 0.007 & 0.044 & 20.427  \\
& RFA    & 0.496 & 0.269 & 69.158  & 0.009 & 0.039 & 33.648   \\
& $\mathrm{AIR_{woATI}}$ & 0.094 &   0.380 & 1.171    & 0.004 & 0.034 & 54.363   \\
& $\mathrm{AIR_{FS}}$  & 0.497   & 0.248 & 22.783  & 0.007 & 0.045 & 25.374   \\
& $\mathrm{AIR_{\beta=10}}$  & 0.647   &  0.212& 50.160  & 0.008& 0.047 & 19.448  \\
& $\mathrm{AIR_{\beta=2}}$  & 0.749   & 0.180 & 75.574   & 0.009& 0.050 & 15.202 \\
& \textbf{$\mathrm{AIR_{\beta=1}}$} & \textbf{0.811} & \textbf{0.159} & \textbf{103.810}  & \textbf{0.010} & \textbf{0.051} & \textbf{11.267} \\

\bottomrule
\end{tabular}%
        
        \label{tb:complete_ablation_GAN}
\end{table}

\begin{table}[H]

  \centering
  \caption{Summary of ablation study results on DiffSwap and Diffface. AIR with and without ATI are denoted as AIR and $\mathrm{AIR_{woATI}}$, respectively. AIR with surrogate models including FS networks is denoted as $\mathrm{AIR_{FS}}$. $\mathrm{AIR_{woATI}}$, $\mathrm{AIR_{FS}}$ and ATI are performed under $\epsilon=0.02$.
  AIA with $\epsilon=0.02$ and $\epsilon=0.05$ are denoted as $\mathrm{AIA_{0.02}}$ and $\mathrm{AIA_{0.05}}$, respectively.}

  \setlength{\tabcolsep}{0.4mm}
   \begin{tabular}{llcccccc}
\toprule
& Method & $\mathrm{ASR_{}}$   $\uparrow$ & COS    $\downarrow$        & $\mathrm{RR_{}}$   $\uparrow$             & LPIPS $\uparrow$ & EMSE $\uparrow$ & AWSS $\downarrow$\\
\midrule

\multirow{9}{*}{\rotatebox[origin=l]{90}{DiffSwap}} & ORI & 0.448& 0.294 & 1.369   & 0.000 & 0.000 & 40.505  \\
& $\mathrm{AIA_{0.02}}$   &	0.887 & 0.163 &	77.322  &	0.012 &	0.056 & 16.498\\
& $\mathrm{AIA_{0.05}}$   &0.954  & 0.129 &	124.860		&0.012	&0.056 & 9.502\\
& RFA   &	0.726 & 0.209&	47.432 &	0.012	& \textbf{0.057} & 27.135\\
& $\mathrm{AIR_{woATI}}$  &	0.681 &0.240 &	16.779&	0.011&	0.056&	31.418\\
& $\mathrm{AIR_{FS}}$ &	0.827 & 0.189&	48.861 &	0.012 &0.056	&24.060\\
& $\mathrm{AIR_{\beta=10}}$ &	0.949 & 0.124&	134.919&	0.012&	0.056&	12.806\\
& $\mathrm{AIR_{\beta=2}}$ &	0.952 & 0.127&	132.913&	0.012&	0.056&	11.201\\
& \textbf{$\mathrm{AIR_{\beta=1}}$} & \textbf{0.970} & \textbf{0.101} & \textbf{193.520}  & \textbf{0.013} & 0.056 & \textbf{7.354} \\
\midrule
\multirow{9}{*}{\rotatebox[origin=l]{90}{Diffface}} & ORI & 0.009& 0.458 & 0.013   & 0.003 & 0.000 & 90.523  \\
& $\mathrm{AIA_{0.02}}$   &	0.462 &0.276&	25.135&	0.010&	0.038&	37.950\\
& $\mathrm{AIA_{0.05}}$    &	0.309 & 0.208&	49.519&	0.011&	0.039	&23.368\\
& RFA   &	0.508 & 0.310&	76.806&	0.013&	0.047&	42.049\\
& $\mathrm{AIR_{woATI}}$ & 0.196 &0.442&	3.578	&	0.010&	0.046&	78.661\\
& $\mathrm{AIR_{FS}}$ &	0.621 &  0.217&	60.012&	\textbf{0.014}&	\textbf{0.047}&	19.873\\
& $\mathrm{AIR_{\beta=10}}$ &	0.746 & 0.189&	75.189&	0.014&	0.047&	18.753\\
& $\mathrm{AIR_{\beta=2}}$ &	0.735 & 0.188&	76.765&	0.014&	0.047&	17.574\\
& \textbf{$\mathrm{AIR_{\beta=1}}$} & \textbf{0.902} & \textbf{0.116} & \textbf{190.082}  & 0.013 & 0.038 & \textbf{7.080} \\

\bottomrule
\end{tabular}%
  
  \label{tb:complete_ablation_diffusion}
\end{table}

\subsection{Results of different surrogate face recognition models.} In Table \ref{tab:fr} in the appendix, we compare 4 different kinds of face recognition models as surrogate models of AIR against MegaGAN, including SphereFace \footnote{https://github.com/wy1iu/sphereface.git}, CosFace \footnote{https://github.com/MuggleWang/CosFace-pytorch.git}, ElasticFace \footnote{https://github.com/fdbtrs/ElasticFace.git}, and ArcFace. ArcFace delivers the best attack performance among all face recognition models, while AIR with the other models still outperforms the baselines. 

\begin{table}[t]
  \footnotesize
  \caption{Comparison of different face recognition models.}
    \centering

  \begin{tabular}{lccccc}
    \toprule
  Model   & $\mathrm{ASR_{}}$   $\uparrow$  & COS    $\downarrow$        & $\mathrm{RR_{}}$   $\uparrow$            & LPIPS $\uparrow$ & EMSE $\uparrow$ \\
  \midrule
  Sphereface   & 0.618   & 0.230  & 32.167    & 0.007 & 0.055   \\
  CosFace & 0.646 & 0.224 & 35.150  & 0.007 & 0.056  \\
  
  
  ElasticFace   & 0.638   & 0.223  & 36.105     & 0.007  & 0.055  \\
  
  
  ArcFace   & \textbf{0.796 }   & \textbf{0.168} & \textbf{96.294}    & \textbf{0.009 } & \textbf{0.059 } \\
  
  \bottomrule
  \end{tabular}
  
    \label{tab:fr}
  \vspace{-0.5cm}
  \end{table}

\subsection{Detail of Surrogate Models of Baselines for Transfer Attack}
Most of the baselines, including DA, TAA, LSA, SAA, CMUA, and TAFIM, rely on surrogate FS models to generate AEs against target FS models. However, these baselines cannot effectively use diffusion-based models as surrogates for two main reasons: 1) they were designed specifically for attacking GAN-based FS models, not diffusion-based ones, and therefore do not account for the unique procedures needed to generate AEs with diffusion-based models, 2) applying these baselines directly to diffusion-based models is challenging due to the need for backpropagation of gradients from the final loss to the initial input across the entire computation graph. Diffusion-based FS models generate final outputs through multi-step inference (often exceeding 50 steps), leading to large computation graphs. Backpropagation through these large graphs incurs significant memory usage, high computational costs, and issues with gradient vanishing or explosion. Therefore, when generating AEs for a target GAN-based model, we use two other GAN-based models as surrogates for the baselines, excluding diffusion-based models from the surrogates. Similarly, for a target diffusion model, we also employ only GAN-based models as surrogates. Note that the target model is not used as a surrogate to maintain the focus on transferability in the attack.

\subsection{Comparison of AIR with Transfer Attacks}

Except for the aforementioned baselines, we also conduct comparisons between AIR and attacks adapted from TI and TAIG \cite{huang2022transferable}, which are state-of-the-art transferable attacks against image classification. To explore the influence of surrogate models simultaneously, we apply TI and TAIG to DA, utilizing either FS or face recognition models as surrogate models. Since MegaGAN, FaceShifter, and SimSwap are relatively more robust than other FS models against transferable attack according to the previous results, we take the three FS models as example and compare AIR with TAIG and TI on the three FS models. When employing face recognition models as surrogate models, we use the same eight open-source face recognition models and the identical identity loss function as described in Eq. \ref{eqn:AIA_objective_function}, mirroring the approach adopted by AIR. We set $\epsilon$ to 0.02 for all attacks (setting 1). For FS models used as surrogate models, we apply TI and TAIG on DA, incorporating the mean square error between the original output and the attack output as the objective function to enhance transferability further. Two FS models serve as surrogate models, while the remaining face swapping model is the target model.  Given that DA with TAIG and DA with TI underperform when $\epsilon=0.02$, we increase $\epsilon$ to 0.05 (setting 2). We employ the same 1000 source images to generate adversarial images. The experimental results are presented in Table \ref{tb:transfer}. Due to budget constraints, we exclude the AWSS metric from this evaluation.

 Based on the findings presented in Table \ref{tb:transfer}, it is evident that attacking the identity extraction process yields greater transferability compared to targeting the swapped output of surrogate FS models in compromising identity information. It highlights that focusing on attacking the identity extraction process is considered more suitable for our objectives, given that our attack's primary goal is to hinder FS models from generating facial images resembling source faces. When comparing the performance of AIR, TAIG, and TI under the same $\epsilon$ value and using the same set of surrogate models and objective functions, AIR consistently outperforms TAIG and TI. When compared to DA with TAIG and TI under $\epsilon=0.05$, AIR outperforms them in nearly all metrics except for the EMSE and LPIPS against MegaGAN. However, the adversarial examples generated by DA with TI and TAIG under $\epsilon=0.05$ are with conspicuous noise patterns as shown in Fig. \ref{Fig:TI_TAIG}. These results demonstrate the high efficacy of AIR in terms of transferability and image quality.

\begin{table}[h!]
    \centering

    \caption{The attack results of the TAIG, TI, and AIR on MegaGAN, SimSwap, and FaceShifter. The methods with the subscript of "FR" are the results of corresponding methods under setting 1. The methods with the subscript of "FS" are the results of corresponding methods under setting 2.}

      \setlength{\tabcolsep}{0.4mm}
		\begin{tabular}{llcccccc}
			\toprule
      & Method  & $\epsilon$  & $\mathrm{ASR_{}}$   $\uparrow$  & $\mathrm{COS_{}}$    $\downarrow$        & $\mathrm{RR_{}}$   $\uparrow$            & LPIPS $\uparrow$ & EMSE $\uparrow$ \\
			\midrule
\multirow{6}{*}{\rotatebox[origin=l]{90}{MegaGAN}} & Original & 0 & 0.016 &  0.452 & 0.018  & 0.001 & 0.000 \\
&    $\mathrm{TAIG_{FS}}$     & 0.05  & 0.510   &   0.255 & 22.723 & \textbf{0.010} & \textbf{0.060} \\
& 	  	$\mathrm{TI_{FS}}$   &  0.05  & 0.447    & 0.266 & 20.013  & 0.009 & 0.060 \\
&          $\mathrm{TAIG_{FR}}$      & 0.02 & 0.468 & 0.274 & 23.475 & 0.004 & 0.051 \\
&        	 	$\mathrm{TI_{FR}}$  &  0.02  & 0.431 &  0.280 & 24.016  & 0.005 & 0.050\\	
&         AIR   &{0.02}  &\textbf{0.796} &\textbf{0.168}	&\textbf{96.294}	  &0.009 &{0.059} \\
			\midrule
\multirow{6}{*}{\rotatebox[origin=l]{90}{SimSwap}} & Original & 0        & 0.007      & 0.492        & 0.007           & 0.004 & 0.000    \\
		 & $\mathrm{TAIG_{FS}}$     & 0.05  & 0.232 & 0.328  & 6.816    & 0.009  & 0.045  \\
     &      $\mathrm{TI_{FS}}$       &0.05 & 0.123     & 0.369      & 1.841   & 0.007 & 0.041 \\
     &  $\mathrm{TAIG_{FR}}$      & 0.02  &	0.599 &0.222 &46.985  &	0.008 &	0.046 \\
     &           $\mathrm{TI_{FR}}$      &0.02 &0.432    & 0.269 	&30.390 		&0.007 	&0.042 \\
     &  AIR     & {0.02}&\textbf{0.699} & \textbf{0.191}&	\textbf{82.384}	 &\textbf{0.010}	&\textbf{0.048} \\ 
			\midrule
  \multirow{6}{*}{\rotatebox[origin=l]{90}{FaceShifter}} & Original & 0 &0.009 &0.458 &0.013  & 0.003 & 0.000                \\
	& 	 $\mathrm{TAIG_{FS}}$     & 0.05  & 0.201  & 0.319  & 3.639   & 0.008  & 0.047 \\
  &         $\mathrm{TI_{FS}}$   &0.05  & 0.261  & 0.318      & 5.122   & 0.008 & 0.047 \\		
  &  $\mathrm{TAIG_{FR}}$      & 0.02  &0.463 &0.269 	&25.655 	 	&0.007 	&0.041\\
  &             $\mathrm{TI_{FR}}$     &0.02  &0.275  &0.335 	&8.136	 	&0.005	&0.035  \\
          
	& 		 AIR    &{0.02} &\textbf{0.811}&\textbf{0.159} &	\textbf{103.810}	 &\textbf{0.010}&\textbf{0.051} \\
			\bottomrule
		\end{tabular}
     
	\label{tb:transfer}

\end{table}

\begin{figure}
\centering
\includegraphics[width=0.15\textwidth]{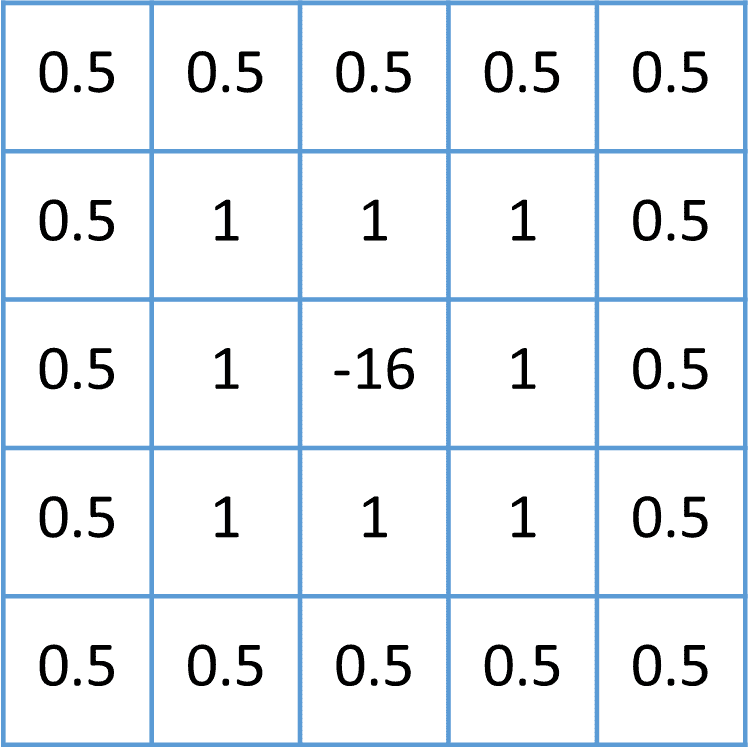}
\caption{The filter.}
	\label{fig:filter}
\end{figure}



 \begin{figure} 
	\centering 
	\includegraphics[width=0.8\columnwidth]{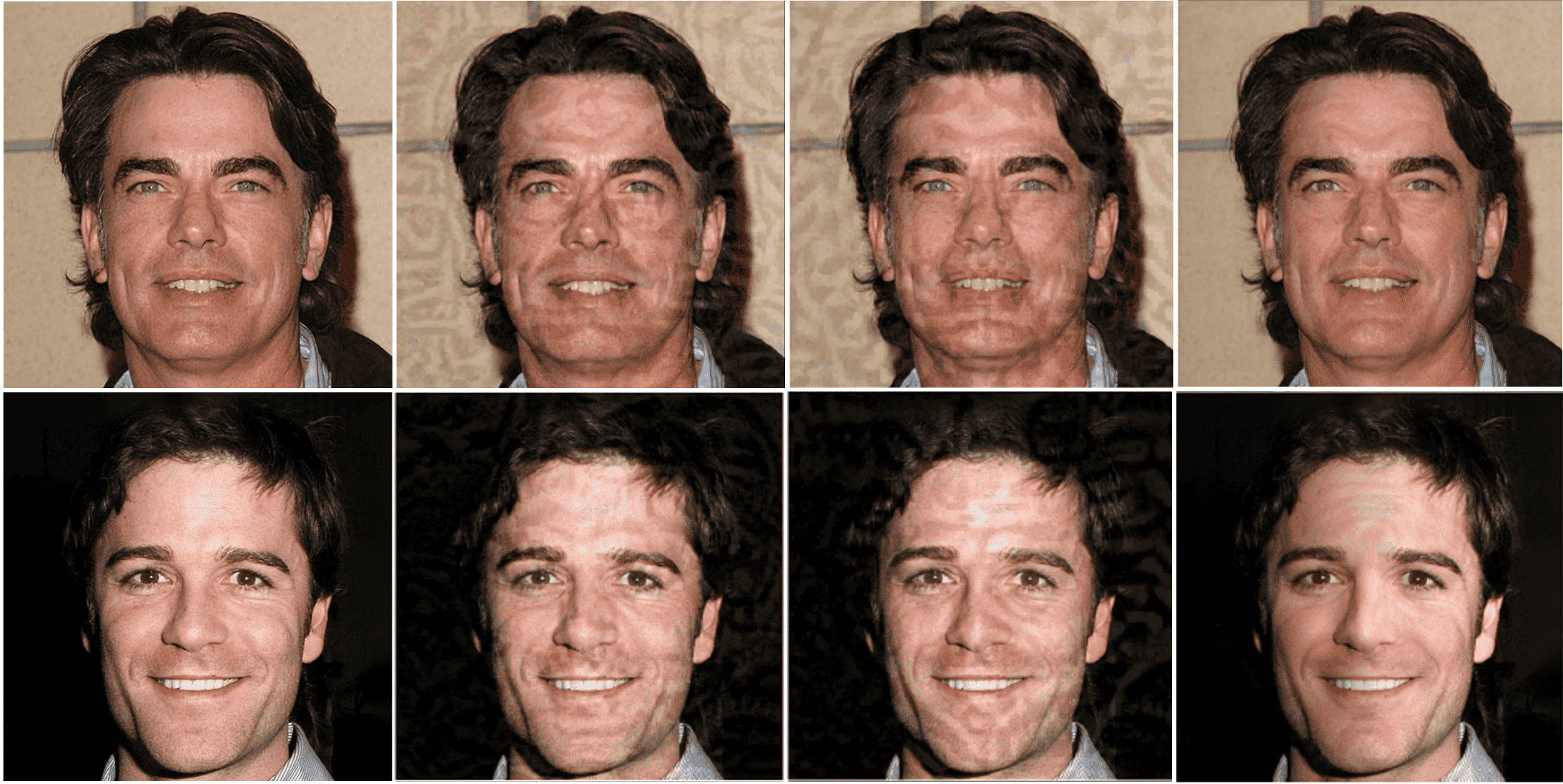} 
	\caption{Adversarial examples generated by TI, TAIG, and AIR. The first column is the original images. The images from the second column to the fourth column are images from TI, TAIG, and AIR respectively.} 
	\label{Fig:TI_TAIG} 
\end{figure}


\subsection{User Study} To further validate the acceptability of the difference in illumination between the adversarial example and the original image, we conducted a user study with approval from the Institutional Review Board (IRB) at Nanyang Technological University, Singapore (IRB-2021-1055). The objective of the study is to assess whether the adversarial perturbations generated by the best three methods are conspicuous and affect image qualities. The user study is conducted in a single room with a 27 inches iMAC monitor to present all the visual stimuli and instructions. The experiment script is coded by MATLAB R2021b with Psychophysics Toolbox Version 3. After the script is ready, we post the recruitment information on Telegram and successfully recruit 72 participants between the age of 20 to 50, including 32 males and 40 females. The procedure of the user study is as follows: 1) in the preparation stage, each participant needs to fill in a consent form and demographic information; 2) in the introduction stage, each participant is required to read the introduction and then a researcher will further elaborate the details of the evaluation to them; 3) in the practice stage, the participants need to evaluate the same four images, composed of two high-quality original images and two attack images with obvious noise selected by us. This step aims to confirm that the participant understands their task in this user study well; 4) in the formal evaluation stage,  each image will appear for 2 seconds, and then the participants have 10 seconds to give a score from 1 to 5 to indicate how they feel about the artifacts on the image.

To eliminate the impact of variable image qualities, we follow the Latin square design and randomly selected 100 images from the original source images and their corresponding adversarial examples generated by DA, CMUA, and AIR as our test dataset. Each participant assesses 100 images containing 25 original and 25 adversarial images from each of the three attacks. Every fourth participant evaluate all 400 images.  The images are rated on a scale of 1-5, where lower scores indicate less perceptible perturbations with less impact on image quality. A score of 4 or 5 indicates conspicuous perturbations that significantly affected image quality. The meanings of each score are provided as follows: 1 means “Certain that there is no artifact”; 2 means “There are slight artifacts on few regions”; 3 means “There are artifacts on the face but do not influence the quality of images”; 4 means “There are strong artifacts on the face region and influence the quality of images”; 5 means “There are strong artifacts on the entire region and influence the quality of images seriously. We collate and compare the ratings of 72 participants, and the results showed that 61.1\% and 74.1\% of the ratings for CMUA and DA are 4 or 5, respectively. However, for AIR, only 34.5\% of the ratings are 4 or 5, which was significantly lower than the other two methods. Notably, even original images have 7.4\% of ratings of 4 and 5. These findings suggest that AIR generates less noticeable adversarial perturbations than DA and CMUA, with minimal impact on image quality.



\begin{figure}[h!]
	\centering 
	\includegraphics[width=0.8\columnwidth]{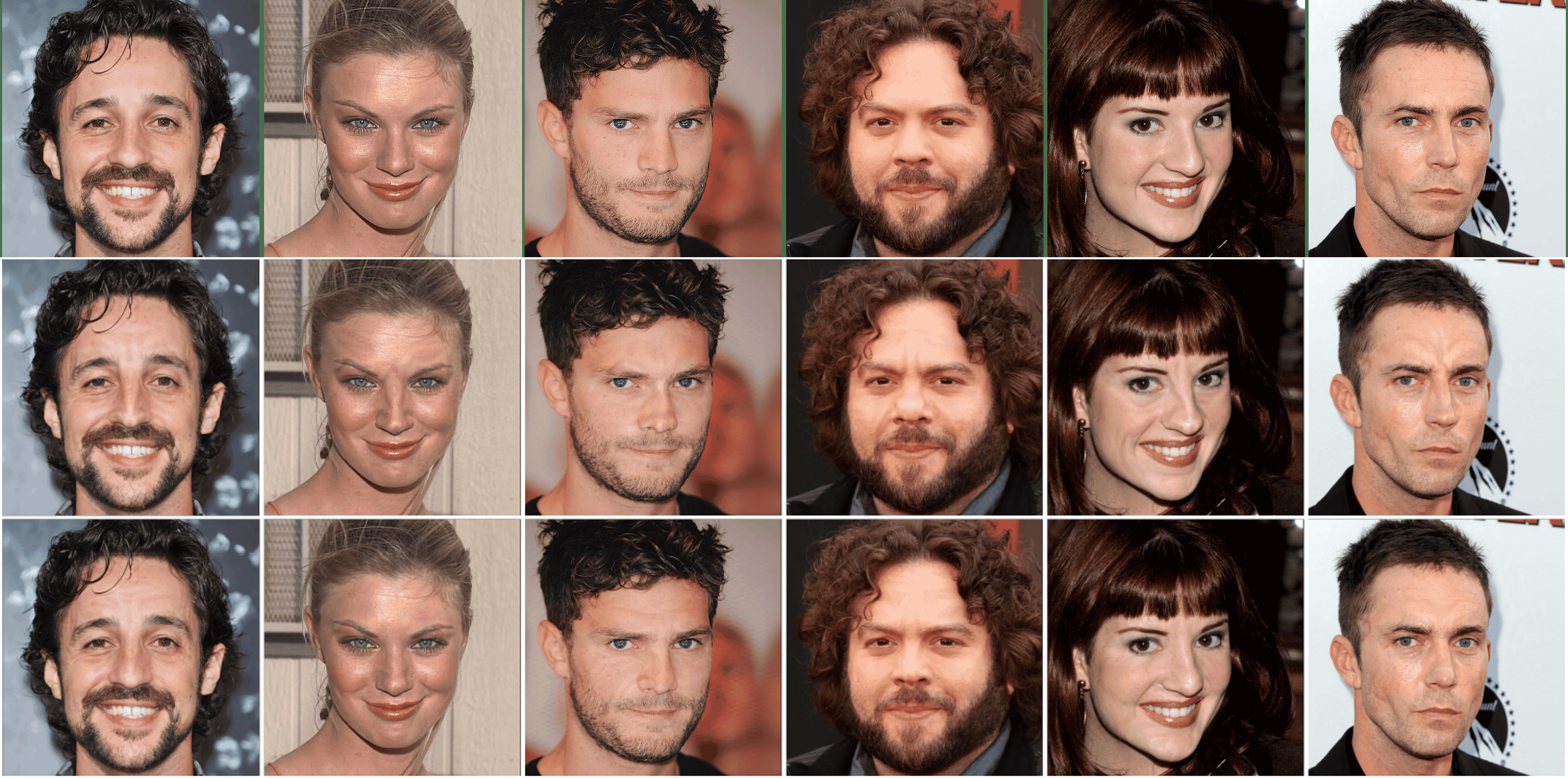} 
	\caption{AIA with ATI under $\epsilon=0.05$ and $\epsilon=0.02$. The first column is the original images. The second and third columns are the adversarial examples generated with $\epsilon=0.02$ and $\epsilon=0.05$, respectively.} 
	\label{fig:AIA_0.02_0.05_0} 
\end{figure}

\begin{figure}[h!] 
	\centering 
	\includegraphics[width=\columnwidth]{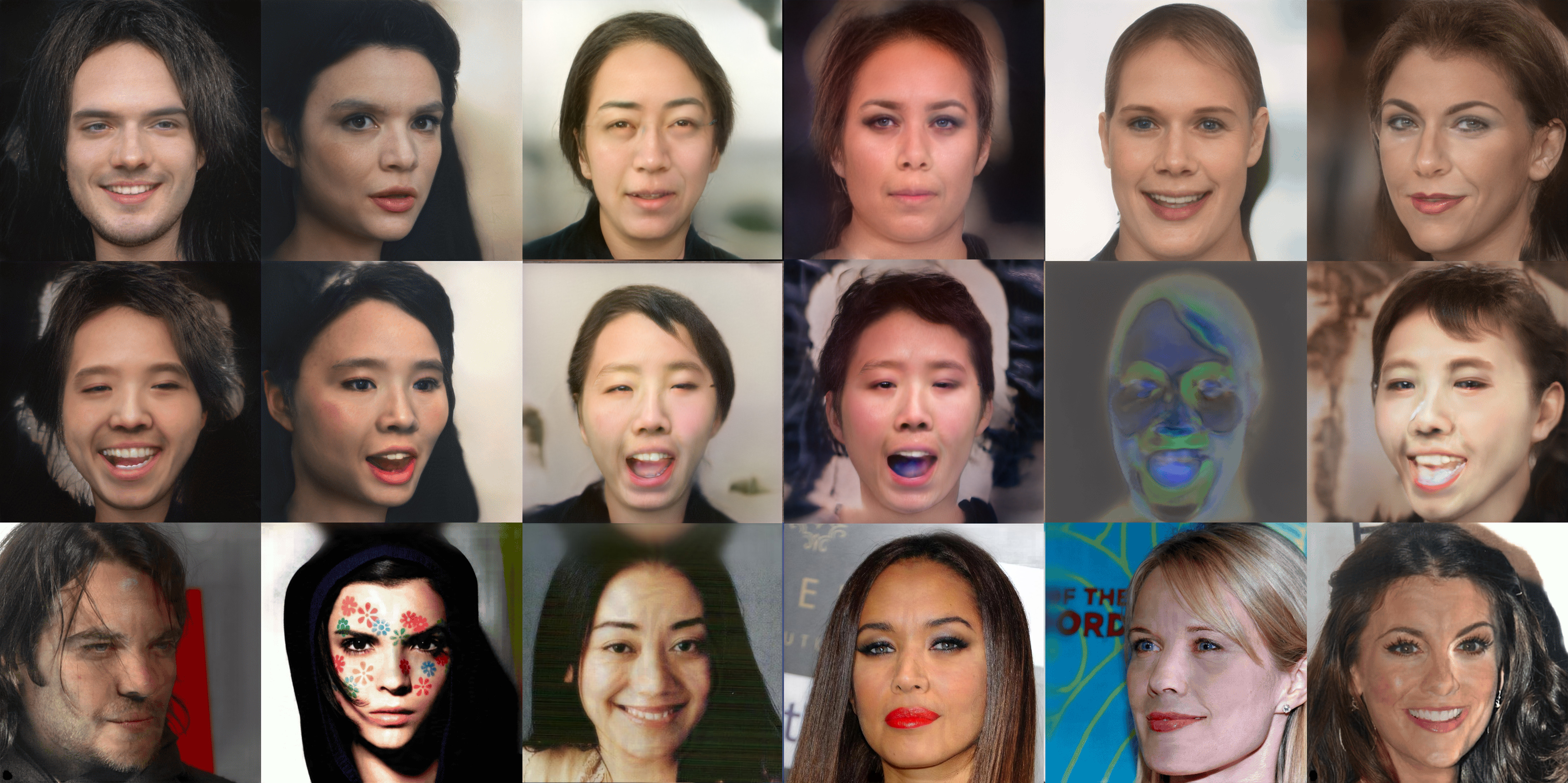}
	\caption{Examples of AIR against MegaGAN. The first row is the original output. The second row is the attacked output and the third row is the attack input of AIR.} 
	\label{Fig.result_mega} 
\end{figure}

\begin{figure}[h!] 
	\centering 
	\includegraphics[width=\columnwidth]{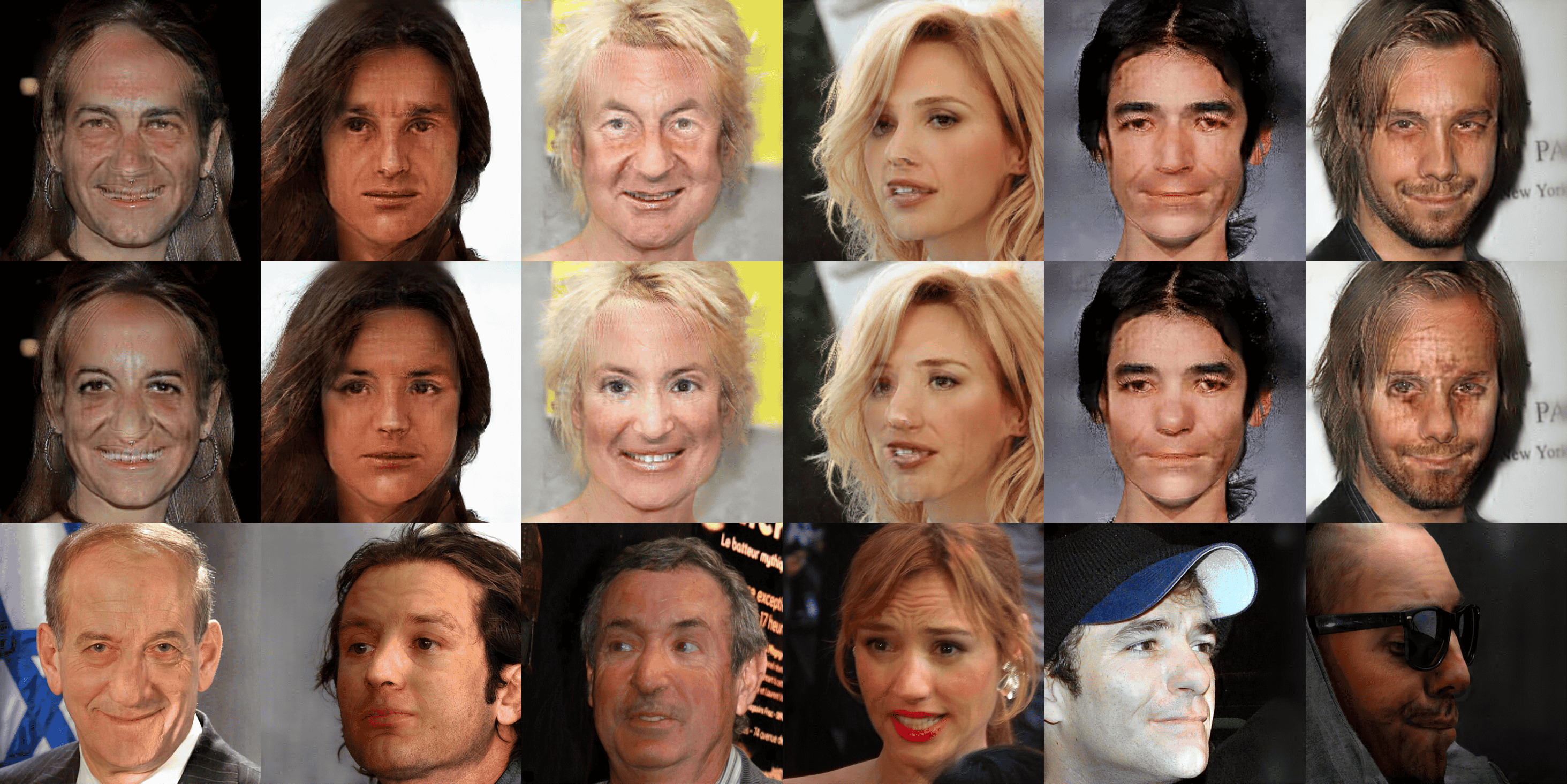}
	\caption{Examples of AIR against FaceShifter. The first row is the original output. The second row is the attacked output and the third row is the attack input of AIR.} 
	\label{Fig.result_mega} 
\end{figure}

\begin{figure}[h!] 
	\centering 
  \includegraphics[width=\columnwidth]{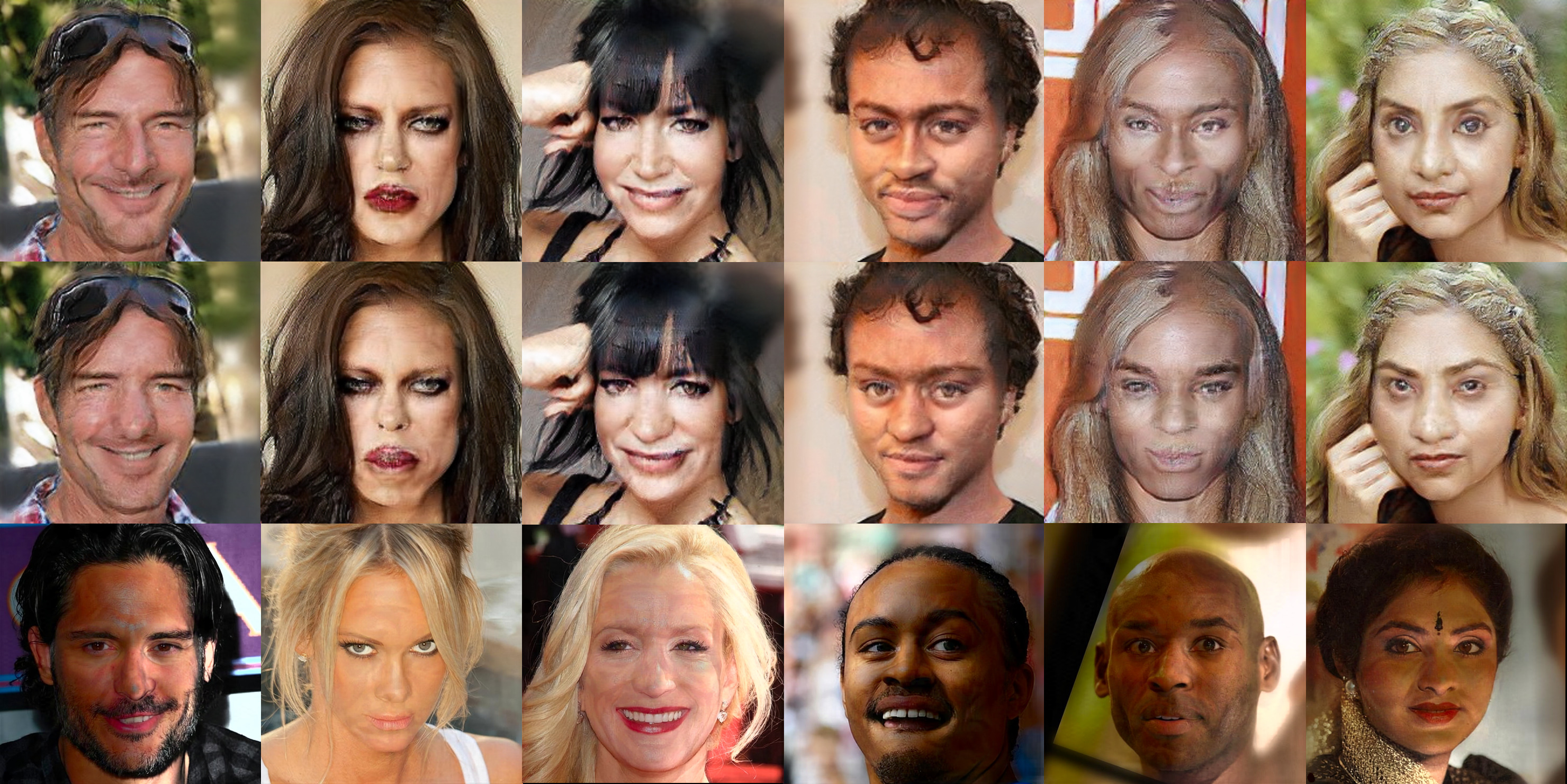}
	\caption{Examples of AIR against SimSwap. The first row is the original output. The second row is the attacked output and the third row is the attack input of AIR.} 
	\label{Fig.result_mega} 
\end{figure}

\begin{table}[H]
  \footnotesize
    \centering
    \caption{Test AIR across different datasets on MegaGAN.}
    
  \setlength{\tabcolsep}{0.4mm}
  \begin{tabular}{lccccc}
    \toprule
  Datasets & $\mathrm{ASR_{}}$   $\uparrow $ & COS    $\downarrow$        & $\mathrm{RR_{}}$   $\uparrow$            & LPIPS $\uparrow$ & EMSE $\uparrow$ \\
  \midrule
  CelebA-HQ    & 0.796   & 0.168 &  96.294    & 0.009  & 0.059   \\
  VGGFace2-HQ & 0.809 &  0.157  & 95.451     &  0.010 & 0.058 \\
  FFHQ &  0.810  & 0.166 & 88.247  & 0.010 &  0.060   \\
  
  \bottomrule
  \end{tabular}
  
    \label{tab:dataset}
  \end{table}

\subsection{Test with Various Test Datasets} 
To further verify the transferability of AIR, we evaluate AIR on three different test datasets: CelebA-HQ, VGGFace2-HQ, and FFHQ. As shown in Table \ref{tab:dataset}, AIR consistently demonstrates excellent performance across all these datasets.

\subsection{Apply AIR against Other Facial Manipulation}
To evaluate AIR's defense against other facial manipulations, we tested it on the PG facial attribute editing model \cite{huang2021initiative}\footnote{The facial attribute editing model is based on StarGAN}. 
We utilized AEs generated by AIR as input to the facial editing model. 
The facial attribute editing model requires a target attribute rather than a target identity. Thus, we randomly selected a target attribute for each input image to transfer its attribute.
The results presented in Table \ref{tb:facial_manipulation} demonstrate that AIR significantly outperforms PG in defending against facial manipulation. Despite not being specifically designed for defense against facial manipulation, AIR effectively safeguards the identities of input images by misleading the facial editor into reconstructing the faces with incorrect identities. This performance improvement is notable compared to PG, highlighting AIR's ability to better protect the identities of input images in the context of facial manipulation.

\begin{table}[t]
   
   \centering

   \caption{The attack results of AIR and PG on the target facial manipulation model. $\epsilon$ is set to 0.02 for both PG and AIR.}

   \setlength{\tabcolsep}{0.4mm}
  \begin{tabular}{lccccc}
    \toprule
    Method   & $\mathrm{ASR_{}}$   $\uparrow$ & COS    $\downarrow$        & $\mathrm{RR_{}}$   $\uparrow$            & LPIPS $\uparrow$ & EMSE $\uparrow$ \\
  \midrule                              
  Original    & 0.000    & 0.880 & 0.000  & 0.001 & 0.036  \\
  $\mathrm{PG}$        & 0.000   & 0.836 & 0.000    & 0.012 & \textbf{0.066}   \\
  \textbf{$\mathrm{AIR}$} & \textbf{0.488} & \textbf{0.262} & \textbf{31.919}   & \textbf{0.019} & 0.052 \\
  \bottomrule 
  \end{tabular}
  
  \label{tb:facial_manipulation}

  \end{table}

\subsection{Transferability in Face Recognition (FR) Model}
Moreover, we also test whether AIR with ArcFace can induce large embedding shifts in AEs, by comparing the identity vectors of original images with their corresponding AEs across other three FR models. The result shows that AIR with ArcFacce reduces cosine similarity by an average of 0.980 on CosFace, 0.979 on SphereFace, and 0.988 on ElasticFace, suggesting high transferability of AIR across different FR models and effectiveness in distorting identical information. This also justifies AIR's capability in attacking the identity extraction modules of FS models.

\subsection{Comparison with Baselines under Basic Image Operation}
The generated adversarial images are first processed by one of the aforementioned operations and then fed into the three FW models. {JPEG compression} compresses the images to the quality level of 50\%. {Resizing} rescales the adversarial examples randomly with a scale factor ranging from $0.8$ to $1.2$ and {rotation} rotates the adversarial examples clockwise with a random angle from $-15^{\circ}$ to $15^{\circ}$. Since MegaGAN, FaceShifter, and SimSwap are relatively more robust than other FS models against AIR according to the previous results, we take the three FS models as example and compare AIR with baselines under basic image operations to verify the robustness of AIR on the three FS models.  The attack results of the baseline methods and AIR under the above-mentioned pre-processing operations are given in Table \ref{tb:preprocess_compress}, \ref{tb:preprocess_resize}, and \ref{tb:preprocess_rotate}. Compared with the results without the pre-processing operations given in Table \ref{tb:complete_compare_megagan}, \ref{tb:complete_compare_simswap}, and \ref{tb:complete_compare_faceshifter}, we note that the pre-processing operations have a very minor impact on all these methods with ATI. With these pre-processing operations, AIR still performs the best among all of them on all the metrics on the three models.

\begin{table*}[ ]
	
  \centering
  \caption{The attack results of the eight baselines and AIR on MegaGAN, SimSwap and FaceShifter. The results of ORI are from the original output without attack. The LPIPS in the ORI row is the LPIPS between the original output and its JPEG-50 image, which are used as references. The results of baseline methods are evaluated by applying ATI on the corresponding methods. The $\uparrow$ indicates that a higher value stands for a stronger attack. The $\downarrow$ indicates that a lower value stands for a stronger attack.} 

\begin{tabular}{clc|ccccc }
\toprule
                             &         &               & \multicolumn{5}{c}{  JEPG Compression}    \\
Model                        & Method  & $\epsilon$  & ASR  $\uparrow$ & COS    $\downarrow$        & RR  $\uparrow$             & LPIPS $\uparrow$ & EMSE $\uparrow$    \\
\midrule
\multirow{10}{*}{\begin{tabular}[c]{@{}c@{}}Mega-\\ GAN\end{tabular}}    & ORI   &0     &0.017 & 0.451 & 0.020	 & 0.001 & 0.000  \\
  & PG      &   0.05         &  0.046  & 0.423    &  0.088     &  0.002    &  0.042       \\
&DA  &0.05    &0.343 & 0.296	&7.856		&0.007	&0.058  \\ 
&TAA  &0.05    &	0.042 & 0.439&	0.814&0.001	&0.019   \\
&LSA  &0.05    &	0.097 & 0.395&	1.640  &0.003	&0.044   \\
&SAA  &0.05    &0.097 & 0.388 &0.502   &0.003	&0.044  \\
&CMUA  &0.05   &0.155 &0.363&	15.688		&0.008	&0.055 \\

& TCA-GAN & 0.05   & 0.015 & 0.449 &	0.159	&	0.001	& 0.030 \\
& TAFIM & 0.05   &	0.012  & 0.462 &	0.016 &	0.001 &	0.019 \\

&AIR  &0.02    &\textbf{0.796}  & \textbf{0.168}&	\textbf{96.294}	 &\textbf{0.009}	&\textbf{0.059}  \\ 
\midrule\multirow{10}{*}{\begin{tabular}[c]{@{}c@{}}Sim-\\ Swap\end{tabular}}      
 & ORI & 0   & 0.011  & 0.492 & 0.014 & 0.003 & 0.000  \\
 & PG & 0.05    &0.021 & 0.453 & 0.032   & 0.002 & 0.024     \\
 & DA & 0.05     & 0.086  & 0.405 & 0.712 & 0.005 & 0.036  \\
 & TAA & 0.05     & 0.013 & 0.482 & 0.035 & 0.001 & 0.008 \\
 & LSA & 0.05   & 0.031 & 0.457 & 0.084 & 0.002 & 0.027  \\
 & SAA & 0.05   & 0.029 & 0.469 & 0.050 & 0.001 & 0.021  \\
 & CMUA & 0.05   & 0.056 & 0.424 & 0.310 & 0.004 & 0.033   \\

 & TCA-GAN & 0.05 &	0.007 & 0.482 &	0.010  &	0.000 &	0.010 \\
 & TAFIM & 0.05  &	0.010 & 0.480 &	0.011  &	0.000	& 0.010 \\

&\textbf{AIR}  &0.02     &\textbf{0.699}  &\textbf{0.191}   &\textbf{82.384}    &\textbf{0.010}	&\textbf{0.048}  \\
\midrule \multirow{10}{*}{\begin{tabular}[c]{@{}c@{}}Face-\\ Shifter\end{tabular}}  
 & ORI & 0    & 0.009 & 0.459 & 0.011 & 0.004 & 0.000   \\
 & PG & 0.05  & 0.027 & 0.419  & 0.067 & 0.002 & 0.029  \\
 & DA & 0.05    & 0.136 & 0.358 & 1.144 & 0.006 & 0.042   \\
 & TAA & 0.05    & 0.007 & 0.455 & 0.008 & 0.000 & 0.007   \\
 & LSA & 0.05   & 0.065 & 0.380 & 0.479 & 0.005 & 0.039  \\
 & SAA & 0.05    & 0.037 & 0.431 & 0.063 & 0.002 & 0.026  \\
 & CMUA & 0.05    & 0.080 & 0.389 & 0.651 & 0.005 & 0.040 \\

 & TCA-GAN & 0.05  &	0.009 & 0.459&	0.009 &	0.001 &	0.015 \\
 & TAFIM & 0.05  &	0.010 & 0.460 &	0.015  &	0.000 &	0.014 \\

&\textbf{AIR}  &0.02     &	\textbf{0.811}  & \textbf{0.159}&	\textbf{103.810}&\textbf{0.010}	&\textbf{0.051}  \\

                             \bottomrule
\end{tabular}

	\label{tb:preprocess_compress}

\end{table*}

\begin{table*}[ ]
	
  \centering
  \caption{The attack results of the eight baselines and AIR on MegaGAN, SimSwap and FaceShifter. The results of ORI are from the original output without attack.  The LPIPS in the ORI row is the LPIPS between the original output and its JPEG-50 image, which are used as references. The results of baseline methods are evaluated by applying ATI on the corresponding methods. The $\uparrow$ indicates that a higher value stands for a stronger attack. The $\downarrow$ indicates that a lower value stands for a stronger attack.} 

\begin{tabular}{clc|ccccc }
\toprule

                            &         &             & \multicolumn{5}{c}{Resizing}       \\
Model                        & Method  & $\epsilon$  & ASR  $\uparrow$ & COS    $\downarrow$        & RR  $\uparrow$            & LPIPS $\uparrow$ & EMSE $\uparrow$     \\
\midrule\multirow{8}{*}{\begin{tabular}[c]{@{}c@{}}Mega-\\ GAN\end{tabular}}
 & ORI & 0 & 0.018 & 0.453 & 0.021  & 0.001 & 0.000     \\
 & PG & 0.05 & 0.027 & 0.447 & 1.111 & 0.001 & 0.031      \\
 & DA & 0.05 & 0.357 & 0.293 & 11.944 & 0.007 & 0.057    \\
 & TAA & 0.05 & 0.034 & 0.441 & 0.472 & 0.001 & 0.008   \\
 & LSA & 0.05 & 0.092 & 0.398 & 2.891 & 0.003 & 0.044    \\
 & SAA & 0.05 & 0.106 & 0.389 & 0.889 & 0.003 & 0.044      \\
 & CMUA & 0.05 & 0.140 & 0.368 & 10.195 & 0.008 & 0.055    \\

 & TCA-GAN & 0.05 & 0.014 & 0.451 &	1.075		& 0.001 &	0.027   \\

 & TAFIM & 0.05 &	0.015 & 0.464 &	0.025&	0.000	&0.007  \\

 & ART & 0.02 & \textbf{0.807}& \textbf{0.165} & \textbf{100.103}  & \textbf{0.010} & \textbf{0.059}      \\
\midrule\multirow{8}{*}{\begin{tabular}[c]{@{}c@{}}Sim-\\ Swap\end{tabular}}     
 & ORI & 0 & 0.005 & 0.493 & 0.006 & 0.004 & 0.000     \\
 & PG & 0.05 & 0.016 & 0.463 & 0.026 & 0.001 & 0.020    \\
 & DA & 0.05 & 0.086 & 0.405 & 0.387 & 0.005 & 0.036    \\
 & TAA & 0.05 & 0.013 & 0.483 & 0.038 & 0.001 & 0.008  \\
 & LSA & 0.05 & 0.027 & 0.458 & 0.069 & 0.002 & 0.026   \\
 & SAA & 0.05 & 0.022 & 0.470 & 0.049 & 0.001 & 0.020    \\
 & CMUA & 0.05 & 0.052 & 0.429 & 0.147 & 0.004 & 0.033    \\

 & TCA-GAN & 0.05 &0.008& 0.482& 0.011	&	0.000&0.009  \\
 & TAFIM & 0.05 &0.011& 0.482	&0.013	&	0.000	&0.010 \\

 & AIR & 0.02 & \textbf{0.701} & \textbf{0.192} & \textbf{79.542}  & \textbf{0.010} & \textbf{0.048}   \\
\midrule\multirow{8}{*}{\begin{tabular}[c]{@{}c@{}}Face-\\ Shifter\end{tabular}}  
 & ORI & 0 & 0.009 & 0.459 & 0.012 & 0.003 & 0.000     \\
 & PG & 0.05 & 0.022 & 0.430 & 0.033 & 0.001 & 0.023       \\
 & DA & 0.05 & 0.136 & 0.358 & 1.112 & 0.006 & 0.042   \\
 & TAA & 0.05 & 0.011 & 0.455 & 0.014 & 0.001 & 0.008      \\
 & LSA & 0.05 & 0.069 & 0.378 & 0.492 & 0.005 & 0.040     \\
 & SAA & 0.05 & 0.035 & 0.431 & 0.056 & 0.002 & 0.026    \\
 & CMUA & 0.05 & 0.074 & 0.391 & 0.475 & 0.005 & 0.039   \\

 & TCA-GAN & 0.05 &	0.010 &0.458 &	0.011  &	0.000 &	0.013 \\

 & TAFIM & 0.05 &0.011& 0.482	&0.013	&	0.000	&0.010  \\

 & AIR & 0.02 & \textbf{0.812} & \textbf{0.156} & \textbf{106.453} & \textbf{0.010} & \textbf{0.051}    \\
                             \bottomrule
\end{tabular}

	\label{tb:preprocess_resize}

\end{table*}

\begin{table*}[ ]
	
  \centering
  \caption{The attack results of the eight baselines and AIR on MegaGAN, SimSwap and FaceShifter. The results of ORI are from the original output without attack. The LPIPS in the ORI row is the LPIPS between the original output and its JPEG-50 image, which are used as references. The results of baseline methods are evaluated by applying ATI on the corresponding methods. The $\uparrow$ indicates that a higher value stands for a stronger attack. The $\downarrow$ indicates that a lower value stands for a stronger attack.} 

\begin{tabular}{clc|ccccc}
\toprule
                            &         &              & \multicolumn{5}{c}{Rotationg}    \\
Model                        & Method  & $\epsilon$  & ASR  $\uparrow$ & COS    $\downarrow$        & RR  $\uparrow$            & LPIPS $\uparrow$ & EMSE $\uparrow$    \\
\midrule\multirow{8}{*}{\begin{tabular}[c]{@{}c@{}}Mega-\\ GAN\end{tabular}}
 & ORI & 0    & 0.043 & 0.424 & 0.072 & 0.001 & 0.000   \\
 & PG & 0.05   & 0.085 & 0.413 & 0.118 & 0.001 & 0.033      \\
 & DA & 0.05     & 0.454 & 0.271 & 16.426 & 0.008 & 0.057 \\
 & TAA & 0.05     & 0.077 & 0.412 & 2.048 & 0.001 & 0.018   \\
 & LSA & 0.05     & 0.112 & 0.374 & 2.947 & 0.003 & 0.043  \\
 & SAA & 0.05   & 0.169 & 0.363 & 0.857 & 0.003 & 0.043   \\
 & CMUA & 0.05    & 0.234 & 0.343 & 17.825 & 0.008 & 0.054   \\

 & TCA-GAN & 0.05   &0.017 &0.449&1.069		&0.001 &	0.028 \\

 & TAFIM & 0.05  &	0.044 & 0.423&	0.062&	0.000&	0.021\\

 & ART & 0.02  & \textbf{0.810} & \textbf{0.168} & \textbf{100.497} & \textbf{0.008} & \textbf{0.058}   \\
\midrule\multirow{8}{*}{\begin{tabular}[c]{@{}c@{}}Sim-\\ Swap\end{tabular}}     
 & ORI & 0   & 0.018 & 0.465 & 0.029 & 0.004 & 0.000  \\
 & PG & 0.05    & 0.023 & 0.411 & 0.034 & 0.002 & 0.026 \\
 & DA & 0.05   & 0.111 & 0.380 & 0.790 & 0.005 & 0.036  \\
 & TAA & 0.05   & 0.037 & 0.456 & 0.128 & 0.001 & 0.008  \\
 & LSA & 0.05    & 0.039 & 0.430 & 0.097 & 0.002 & 0.026 \\
 & SAA & 0.05   & 0.038  & 0.443 & 0.113 & 0.001 & 0.020  \\
 & CMUA & 0.05    & 0.078 & 0.398 & 0.336 & 0.004 & 0.033  \\

 & TCA-GAN & 0.05  &	0.010 & 0.481	&0.013&0.000&0.010 \\
 & TAFIM & 0.05  &	0.015  & 0.450&	0.024&	0.001&	0.014\\

 & AIR & 0.02   & \textbf{0.758} & \textbf{0.175} & \textbf{91.010} & \textbf{0.011} & \textbf{0.048} \\
\midrule\multirow{8}{*}{\begin{tabular}[c]{@{}c@{}}Face-\\ Shifter\end{tabular}}  
 & ORI & 0  & 0.018 & 0.439 & 0.049 & 0.003 & 0.000   \\
 & PG & 0.05    & 0.031 & 0.443 & 0.031 & 0.002 & 0.023   \\
 & DA & 0.05    & 0.191 & 0.341 & 2.081 & 0.006 & 0.042  \\
 & TAA & 0.05    & 0.023 & 0.437 & 0.040 & 0.000 & 0.007  \\
 & LSA & 0.05    & 0.115 & 0.364 & 1.059 & 0.005 & 0.038   \\
 & SAA & 0.05     & 0.048 & 0.415 & 0.222 & 0.002 & 0.025  \\
 & CMUA & 0.05  & 0.107 & 0.372 & 0.612 & 0.005 & 0.040   \\

 & TCA-GAN & 0.05 &	0.010  & 0.458	&0.012&	0.000	&0.014\\

 & TAFIM & 0.05   &0.009  & 0.460	&0.013		&0.000	&0.014\\

 & AIR & 0.02 &  \textbf{0.838} & \textbf{0.151} & \textbf{117.003} & \textbf{0.010} & \textbf{0.051}   \\
                             \bottomrule
\end{tabular}

	\label{tb:preprocess_rotate}

\end{table*}

\end{document}